\documentclass{article}

\usepackage{microtype}
\usepackage{graphicx}
\usepackage{subfigure}
\usepackage{booktabs} % for professional tables
\usepackage{fancyhdr}
\usepackage{lastpage}
\usepackage{extramarks}
\usepackage{chngpage}
\usepackage{soul,color}
\usepackage{graphicx,float,wrapfig}
\usepackage{ifpdf}
\usepackage{indentfirst}
\usepackage{galois}
\usepackage{listings}
\usepackage{enumitem}
\usepackage{amsmath}
\usepackage{amsfonts}
\usepackage{amsthm}
\usepackage{amssymb}
\usepackage{bm}
\usepackage{hyperref}
\usepackage{cleveref}
\usepackage{multicol}

\usepackage[accepted]{icml2020}

\makeatletter
\newif\if@restonecol
\makeatother

\usepackage[algo2e,ruled,vlined,linesnumbered]{algorithm2e}

\newcommand*{\defeq}{\stackrel{\text{def}}{=}}

\newcommand{\poly}{\mathrm{poly}}
\newcommand{\E}{\mathop{\mathbb{E}}}

\newcommand{\Z}{\mathbb{Z}}
\newcommand{\hv}{\hat{v}}
\newcommand{\htv}{\hat{\bm{v}}}
\newcommand{\htheta}{\hat{\theta}}
\newcommand{\bv}{\bar{v}}

\newcommand{\calQ}{\mathcal{Q}}
\newcommand{\Exp}{\textsc{Exploration}}
\newcommand{\optS}{S^{\star}}
\newcommand{\optt}{\theta^{\star}}
\newcommand{\ind}{\mathbb{I}} % indicator function
\newcommand{\Talgo}{\text{ESUCB}\xspace} % trisection algorithm

\newcommand{\chk}{{\textsc{Check}}\xspace}
\newcommand{\tmax}{t_{\rm max}}
\newcommand{\true}{{\sf true}}
\newcommand{\false}{{\sf false}}

\newcommand{\abs}[1]{\left| #1 \right|}

\newcommand{\regret}{{\rm Reg}}
\newcommand{\expect}[1]{\mathbb{E}\left[#1\right]}

\newcommand{\given}{\;\middle|\;}
\newcommand{\Update}{\textsc{Update}}
\newcommand{\cT}{\mathcal{T}}

\newcommand{\event}{{\cal E}}

\newcommand{\AS}{\Psi^{(\mathrm{asst})}}
\newcommand{\IS}{\Psi^{(\mathrm{item})}}

\newtheorem{theorem}{Theorem}
\newtheorem{lemma}[theorem]{Lemma}

\newtheorem{fact}[theorem]{Fact}
\newtheorem{observation}[theorem]{Observation}

\usepackage{natbib}
\begin{document}

\twocolumn[
\icmltitle{Multinomial Logit Bandit with Low Switching Cost}

% It is OKAY to include author information, even for blind
% submissions: the style file will automatically remove it for you
% unless you've provided the [accepted] option to the icml2019
% package.

% List of affiliations: The first argument should be a (short)
% identifier you will use later to specify author affiliations
% Academic affiliations should list Department, University, City, Region, Country
% Industry affiliations should list Company, City, Region, Country

% You can specify symbols, otherwise they are numbered in order.
% Ideally, you should not use this facility. Affiliations will be numbered
% in order of appearance and this is the preferred way.
\icmlsetsymbol{equal}{*}

\begin{icmlauthorlist}
\icmlauthor{Kefan Dong}{kd}
\icmlauthor{Yingkai Li}{yl}
\icmlauthor{Qin Zhang}{qz}
\icmlauthor{Yuan Zhou}{yz}
\end{icmlauthorlist}

\icmlaffiliation{kd}{Institute for Interdisciplinary Information Sciences, Tsinghua University, Beijing, China.}
\icmlaffiliation{yl}{Department of Computer Science, Northwestern University, Evanston, Illinois, USA.}
\icmlaffiliation{qz}{Computer Science Department,
Indiana University, Bloomington, Indiana, USA.}
\icmlaffiliation{yz}{Department of ISE, University of Illinois at Urbana-Champaign, Urbana, Illinois, USA}

\icmlcorrespondingauthor{Yuan Zhou}{yuanz@illinois.edu}

% You may provide any keywords that you
% find helpful for describing your paper; these are used to populate
% the "keywords" metadata in the PDF but will not be shown in the document
\icmlkeywords{Machine Learning, ICML}

\vskip 0.3in
]

% this must go after the closing bracket ] following \twocolumn[ ...

% This command actually creates the footnote in the first column
% listing the affiliations and the copyright notice.
% The command takes one argument, which is text to display at the start of the footnote.
% The \icmlEqualContribution command is standard text for equal contribution.
% Remove it (just {}) if you do not need this facility.

\printAffiliationsAndNotice{Author names are listed in alphabetical order. }  % leave blank if no need to mention equal contribution
%\printAffiliationsAndNotice{\icmlEqualContribution} % otherwise use the standard text.

\begin{abstract}
We study multinomial logit bandit with limited adaptivity, where the algorithms change their exploration actions as infrequently as possible when achieving almost optimal minimax regret. We propose two measures of adaptivity: the assortment switching cost and the more fine-grained item switching cost. We present an anytime algorithm (AT-DUCB) with $O(N \log T)$ assortment switches, almost matching the lower bound $\Omega(\frac{N \log T}{ \log \log T})$. In the fixed-horizon setting, our algorithm FH-DUCB incurs $O(N \log \log T)$ assortment switches, matching the asymptotic lower bound. We also present the ESUCB algorithm with item switching cost $O(N \log^2 T)$.
\end{abstract}

\allowdisplaybreaks

\section{Introduction}
\label{sec:intro}

The dynamic assortment selection problem with the multinomial logic (MNL) choice model, also called MNL-bandit, is a fundamental problem in online learning and operations research.  In this problem we have $N$ distinct items, each of which is associated with a known reward $r_i$ and an {\em unknown} preference parameter $v_i$.  In the MNL choice model, given a subset $S \subseteq [N] \defeq \{1, 2, 3, \dots, N\}$, the probability that a user chooses $i \in S$ is given by  
\begin{equation}
\label{eq:MNL}
p_i(S) = 
\left\{
\begin{array}{lr}
    \dfrac{v_i}{v_0 + \sum_{j \in S} v_j} & \text{if}~i \in S \cup \{0\} \\
    0 & \text{otherwise}
\end{array}
\right. ,
\end{equation}
where ``$0$'' stands for the case that the user does not choose any item, and $v_0$ is the associated preference parameter. As a convention (see, e.g.\ \citealp{AAGZ19}), we assume that no-purchase is the most frequent choice, which is very natural in retailing.  W.l.o.g., we assume $v_0 = 1$, and $v_i \le 1$ for all $i \in [N]$.  The {\em expected reward} of the set $S$ under the preference vector $\bm{v} = \{v_0, v_1, \ldots, v_N\}$ is defined to be 
\begin{equation}
R(S, \bm{v})  = \sum_{i \in S} r_i p_i(S) = \sum_{i \in S} \dfrac{r_i v_i}{1 + \sum_{j \in S} v_j}.
\end{equation}
For any online policy that selects a subset $S_t \subseteq [N]$ \ ($\abs{S_t} \le K$, \text{where $K$ is a predefined capacity parameter}) at each time step $t$, observes the user's choice $a_t$ to gradually learn the preference parameters $\{v_i\}$, and runs for a horizon of $T$ time steps, we define the \emph{regret} of the policy to be
\begin{equation}
\label{eq:regret}
\regret_T \defeq \sum_{t=1}^{T}\left(R(S^\star,\bm{v})- R(S_t,\bm{v})\right),
\end{equation}
where $S^\star = \arg \max_{S \subseteq [N], \abs{S} \le K} R(S, \bm{v})$ is the optimal assortment in hindsight. The goal is to find a policy to minimize the expected regret $\E[\regret_T]$ for all MNL-bandit instances.

To motivate the definition of the MNL-bandit problem, let us consider a fast fashion retailer such as Zara or Mango.  Each of its product corresponds to an item in $[N]$, and by selling the $i$-th item the retailer takes a profit of $r_i$.  At each specific time in each of its shops, the retailer can only present a certain number of items (say, at most $K$) on the shelf due to the space constraints. As a consequence, customers who visit the store can only pick items from the presented assortment (or, just buy nothing which corresponds to item $0$), following a choice model.  There has been a number of choice models being proposed in the literature (see, e.g., \cite{Train09,Luce12} for overviews), and the MNL model is arguably the most popular one. The retailer certainly wants to maximize its profit by identifying the best assortment $S^\star$ to present. However, it does not know in advance customers' preferences to items in $[N]$ (i.e., the preference vector $\bm{v}$), to get which it has to learn from customers' actual choices.  More precisely, the retailer needs to develop a policy to choose at each time step $t$ an assortment $S_t \subseteq [N]\ (\abs{S_t} \le K)$ based on the previous presented assortments $S_1, \ldots, S_{t-1}$ and customers' choices in the past $(t-1)$ time steps.  The retailer's  expected reward in a time horizon $T$ can be expressed by $\sum_{t=1}^T R(S_t, \bm{v})$, which is typically reformulated as the regret compared with the best policy in the form of (\ref{eq:regret}). 

The MNL-bandit problem has attracted quite some attention in the past  decade~\cite{RSS10,SZ13,AAGZ16,AAGZ17,CW18}.  However, all these works do not consider an important practical issue for regret minimization: in reality it is often impossible to {\em frequently} change the assortment display.  For example, in retail stores it may not be possible to change the display in the middle of the day, not mentioning doing it after each purchase.  We thus hope to minimize the number of assortment switches in the selling time horizon {\em without} increasing the regret by much.  Another advantage of achieving a small number of assortment switches is that such algorithms are easier to parallelize, which enables us to learn users' preferences much faster.  This feature is particularly useful in applications such as online advertising where it is easy to show the same assortment (i.e., a set of ads) in a large amount of end users' displays simultaneously.

We are interested in two kinds of switching costs under a time horizon $T$.  The first is the {\em assortment switching cost}, defined as
\begin{equation*}
%\label{eq:asst}
\AS_T \defeq \sum_{t=1}^T \ind[S_t\neq S_{t+1}].
\end{equation*}
The second is the {\em item switching cost}, defined as 
\begin{equation*}
%\label{eq:item}
\IS_T \defeq \sum_{t=1}^T \left| S_t \oplus S_{t+1} \right|,
\end{equation*}
where binary operator $\oplus$ computes the symmetric difference of the two sets. In comparison, the item switching cost is more fine-grained and put less penalty if two neighboring assortments are ``almost the same''. As a straightforward observation, we always have that
\begin{align}
\AS_T \le \IS_T \le \min\{2K, N\} \cdot \AS_T. \label{eq:AS-IS-relation}
\end{align}

\paragraph{Our results.}  
In this paper we obtain the following results for MNL-bandit with low switching cost.  By default all $\log$'s are of base $2$.

We first introduce an algorithm, AT-DUCB, that achieves almost optimal regret (up to a logarithmic factor) and incurs an assortment switching cost of $O(N \log T)$; this algorithm is \emph{anytime}, i.e., it does {\em not} need to know the time horizon $T$ in advance.  We then show that the AT-DUCB algorithm achieves almost optimal assortment switching cost. In particular, we prove that every anytime algorithm that achieves almost optimal regret must incur an assortment switching cost of at least $\Omega(N \log T / \log \log (NT))$.  These results are presented in Section~\ref{sec:AS-logT}.

When the time horizon is known beforehand, we obtain an algorithm, FH-DUCB, that achieves almost optimal regret (up to a logarithmic factor) and incurs an assortment switching cost of $O(N \log\log T)$. We also prove the optimality of this switching cost by establishing a matching lower bound.  See Section~\ref{sec:AS-loglogT}.

For item switches, while the trivial application of \eqref{eq:AS-IS-relation} leads to $O(N^2 \log T)$ and $O(N^2 \log \log T)$ item switching cost bounds for AT-DUCB and FH-DUCB respectively, in Section~\ref{sec:IS}, we design a new algorithm, ESUCB, to achieve an item switching cost of $O(N \log^2 T)$.  In Appendix~\ref{app:ESUCB-further-improve}, we show that a more careful modification to the algorithm further improves the item switching cost to $O(N \log T)$.
%we prove that AT-DUCB incurs $O((N \log T)^{1.5})$ item switching cost, an improvement when $T \ll e^{N}$. We also obtain an algorithm that achieves almost optimal regret and incurs an item switching cost of $O(N \log^2 T)$.  These results are presented in Section~\ref{sec:IS}.

We make two interesting observations from the results above: (1) there is a {\em separation} between the assortment switching complexities when knowing the time horizon $T$ and when not; in other words, the time horizon $T$ is useful for achieving a smaller assortment switching cost; (2) the item switching cost is only at most a logarithmic factor higher than the assortment switching cost.

\paragraph{Technical contributions.}
We combine the epoch-based offering algorithm for MNL-bandits \cite{AAGZ19} and a natural delayed update policy in the design of AT-DUCB. Although a similar delayed update rule has been recently analyzed for multi-armed bandits and Q-learning~\cite{BXJW19}, and such a result does not seem surprising, we present it in the paper as a warm-up to help the readers get familiar with a few algorithmic techniques commonly used for the MNL-bandit problem.

Our first main technical contribution comes from the design of FH-DUCB algorithm, where we invent a novel delayed update policy that uses the horizon information to improve the switching cost from $O(N \log T)$ to $O(N \log \log T)$.  We note that for the ordinary multi-armed bandit problem, recent works \cite{GHRZ19} and \cite{simchi2019phase} managed to show a similar $O(N \log \log T)$ switching cost with known horizon. However, their update rules do not have to utilize the learned parameters for the arms, and a straightforward conversion of such update rules to the MNL-bandit problem does not produce the desired guarantees. In contrast, our update rule, formally described in \eqref{eq:calP}, carefully exploits the structure of the MNL-bandits and uses the information of the partially learned preference parameters (more specifically, $\hat{v}_{i, \tau_i}$ in \eqref{eq:calP}) to adaptively decide when to switch to a different assortment.

Our second main technical contribution is the ESUCB algorithm for the low item switching cost. The technical challenge here stems from the fact that the low item switching cost is a much stronger requirement than the low assortment switching cost, and simple lazy updates with the doubling trick and the straightforward analysis will show that the item switching cost is at most $N$ times the assortment switching cost (see \eqref{eq:AS-IS-relation}), leading to a total item switching cost of $O(N^2 \log T)$. To reducing the extra factor $N$, we propose the idea of decoupling the learning for the optimal revenue and the assortment, so that the offering of the assortment is decided via optimizing a new objective function based on the (usually) fixed revenue estimate.  Since the revenue estimates are fixed, the offered assortments enjoy improved stability, and the item switching cost can be upper bounded by careful analysis.

 %In the outer learning loop, the algorithm learns the optimal revenue with an exponential sequence of learning strides. In the inner learning loop, for the fixed revenue estimate, the algorithm learns the best assortment. Since the revenue estimates are fixed, the offered assortments enjoys improved stability, and the number of item switches is therefore upper bounded.
 
We remark that the item switching cost is a particularly interesting goal that arises in online learning problems when the actions are sets of elements, which is very different from traditional MAB and linear bandits. Thanks to our novel technical ingredients, we are able to bring the item switching cost down to almost the same order as the assortment switching cost. We  hope our results will inspire future study of the switching costs in both settings for other online learning problems with set actions.

\paragraph{Related work.}
MNL-bandit was first studied in \cite{RSS10} and \cite{SZ13}, where the authors took the ``explore-then-commit'' approach, and proposed algorithms with regret $O(N^2 \log^2 T)$ and $O(N \log T)$ respectively under the assumption that the gap between the best and second-to-the-best assortments is known. \cite{AAGZ16} removed this assumption using a UCB-type algorithm, which achieves a regret of $O(\sqrt{NT \log T})$.  An almost tight regret lower bound of $\Omega(\sqrt{NT})$ was later given by \cite{CW18}. \cite{AAGZ17} proposed an algorithm using Thompson Sampling, which achieves comparable regret bound to the UCB-type algorithms while demonstrates a better numerical performance.

Learning with low policy switches (also called learning in the {\em batched model} or {\em limited adaptivity}) has recently been studied in reinforcement learning for several other problems, including stochastic multi-armed bandits~\cite{PRCS15,JJNZ16,AAAK17,GHRZ19,EKMM19,simchi2019phase}, Q-learning~\cite{BXJW19}, and online-learning~\cite{CDS13}.  This research direction is motivated by the fact that in many practical settings, the change of learning policy is very costly.  For example, in clinical trials, every treatment policy switch would trigger a separate approval process. In crowdsourcing, it takes time for the crowd to answer questions, and thus a small number of rounds of interactions with the crowd is desirable. 
% In personalized recommendation systems, it is computationally infeasible to change the ranking policy every time a new data item is received.  
The performance of the learning would be much better if the data is processed in batches and during each batch the learning policy is fixed.

%\paragraph{Asymptotic notations.} For two sequences $\{a_n\}$ and $\{b_n\}$, we write $a_n=O(b_n)$ or $a_n\lesssim b_n$ if there exists a \emph{universal} constant $C<\infty$
%such that $\limsup_{n\to\infty} |a_n|/|b_n|\leq C$.
%Similarly, we write $a_n=\Omega(b_n)$ or $a_n\gtrsim b_n$ if there exists a \emph{universal} constant $c>0$ such that $\liminf_{n\to\infty} |a_n|/|b_n| \geq c$.
%We write $a_n=\Theta(b_n)$ or $a_n\asymp b_n$ if both $a_n\lesssim b_n$ and $a_n\gtrsim b_n$ hold.
%In asymptotic notations, we will drop base notations of logarithms and use instead $\log x$ for both $\ln x,\log_2 x$ as well as logarithms with other constant base numbers.
%In non-asymptotic scenarios, however, base notations will not be dropped and $\ln x$ refers specifically to $\log_e x$.

\section{Warm-up: An anytime algorithm with $O(N \log T)$ assortment switches} \label{sec:AS-logT}

As a warm-up, we begin with a simple anytime algorithm using at most $O(N \log T)$ assortment switches. Our algorithm combines the epoch-based offering framework introduce in \cite{AAGZ16} and a deferred update policy. We will first briefly explain the epoch-based offering procedure, and then present and analyze our algorithm.

\paragraph{The epoch-based offering.} In the epoch-based offering framework, whenever we are to offer an assortment $S$, instead of offering it for only one time period, we keep offering $S$ until a no-purchase decision (item $0$) is observed, and refer to all the consecutive time periods involved in this procedure as an \emph{epoch}. The detailed offering procedure is described in Algorithm~\ref{alg:exploration}, where $t$ is the global counter for the time period, and $\{\Delta_i\}$ records the number of purchases made for each item $i$ in the epoch.

\begin{algorithm2e}[h]
	\caption{$\Exp(S)$}
	\label{alg:exploration}
		Initialize: $\Delta_i\gets 0$ for all $i\in [N]$\;
		\While{\textsc{true}} {
			$t\gets t+1$\;
			Offer assortment $S$, and observe purchase decision $a_t$\;
			\textbf{If} $a_t = 0$ \textbf{then} \textbf{return}  $\{\Delta_i\}$\;
%			\If {$c=0$}{
%				\Return
%			}
			$\Delta_{a_t}\gets \Delta_{a_t}+1$;
		}
\end{algorithm2e}

The following key observation for $\Exp(S)$ states that $\{\Delta_i\}$ forms an unbiased estimate for the utility parameters of all items in $S$.
\begin{observation}\label{ob:unbiased-est}
Let $\{\Delta_i\}$ be returned by $\Exp(S)$. For each $i \in S$, $\Delta_i$ is an independent geometric random variable  with mean $v_i$. Moreover, one can verify that  $\E[\Delta_i] = v_i$ and
\[
\Pr[\Delta_i = k] = \left(\frac{v_i}{1 + v_i}\right)^k \left(\frac{1}{1 +v_i}\right), \forall k \in \mathbb{N}.
\]
\end{observation}

At any time of the algorithm when an epoch has ended, for each item $i \in [N]$, we let $\bar{v}_i = n_i  /T_i$ where $T_i$ is the number of the past epochs in which $i$ is included in the offered assortment, and $n_i$ is the total number of purchases for item $i$ during all past epochs. By Observation~\ref{ob:unbiased-est}, we know that $\bar{v}_i$ is also an unbiased estimate of $v_i$. In \cite{AAGZ16}, the following upper confidence bound (UCB) is constructed for each $i \in [N]$,
\begin{align}\label{eq:vi-ucb}
 \hv_i = \bv_i+\sqrt{\frac{48\bv_i\ln(\sqrt{N}\ell+1)}{T_i}}+\frac{48\ln(\sqrt{N}\ell+1)}{T_i}.
\end{align}
We will compute the assortment for the next epoch based on the vector of UCB values $\htv = (\hv_1, \hv_2, \dots, \hv_n)$.

%\paragraph{The anytime algorithm with deferred UCB updates.} 

\begin{algorithm2e}[h]
	\caption{Anytime Deferred Update UCB (AT-DUCB)}
	\label{alg:doubling-ucb}
		Initialize: $\hv_i\gets1, T_i\gets0$ for all $i\in [N]$, $t\gets 0$\;
		\For{$\ell\gets 1,2,3,\dots,$} {
			Compute $S_\ell=\mathop{\arg\max}_{S\subseteq [N] : |S| \leq K} R(S,\htv)$\; \label{line:optimization}
			$\{\Delta_i\}\gets \Exp(S)$\;
			\For{$i\in S$} {
				$n_i\gets n_i+\Delta_i$ and $T_i\gets T_i+1$\;
				\If{$T_i=2^{k}$ for some $k\in \Z$\label{line:doubling-if}}{
					$\bv_i\gets n_i/T_i$;
					$\hv_i\gets \min\big\{\hv_i,\bv_i+\sqrt{\frac{48\bv_i\ln(\sqrt{N}\ell+1)}{T_i}}+\frac{48\ln(\sqrt{N}\ell+1)}{T_i}\big\}$\;\label{line:doubling}
				}
			}
		}
\end{algorithm2e}

We describe our algorithm in Algorithm~\ref{alg:doubling-ucb}, which can be seen as an adaptation of the one in \cite{AAGZ16}. The main difference from \cite{AAGZ16} is that the UCB values (and hence the assortment) is updated only when $T_i$ reaches an integer power of $2$ for any item $i \in [N]$. This deferred update strategy is implemented in Line~\ref{line:doubling-if}. Also note that instead of directly evaluating \eqref{eq:vi-ucb}, the update in Line~\ref{line:doubling} makes sure that $\hv_i$ is non-increasing as the algorithm proceeds. We comment that the optimization task in Line~\ref{line:optimization} can be done efficiently, as studied in, for example, \cite{RSS10}.

\begin{theorem}\label{thm:ucb-doubling-regret}
For any time horizon $T$, the expect regret incurred by Algorithm~\ref{alg:doubling-ucb} is 
\[
\E\left[\regret_T\right]\lesssim \sqrt{NT \log T},
\]
and the expected number of assortment switches $\E[\AS_T]$ is  $O(N \log T)$. \footnote{For two sequences $\{a_n\}$ and $\{b_n\}$, we write $a_n=O(b_n)$ or $a_n\lesssim b_n$ if there exists a \emph{universal} constant $C<\infty$ such that $\limsup_{n\to\infty} |a_n|/|b_n|\leq C$. Similarly, we write $a_n=\Omega(b_n)$ or $a_n\gtrsim b_n$ if there exists a \emph{universal} constant $c>0$ such that $\liminf_{n\to\infty} |a_n|/|b_n| \geq c$.}
\end{theorem}

The proof of the regret upper bound in Theorem~\ref{thm:ucb-doubling-regret} is similar to that of \cite{AAGZ16}, except for a more careful analysis about the deferred update rule. For completeness, we prove  this part in Appendix~\ref{app:proof-of-theorem-AT-DUCB}.
\begin{proof}[Proof of the assortment switch upper bound in Theorem~\ref{thm:ucb-doubling-regret}]
Let $\mathcal{D}_i^{(\ell)}$ be the event that Line~\ref{line:doubling} is executed in Algorithm~\ref{alg:doubling-ucb} for item $i$ at the $\ell$-th epoch. Recall that the assortment $S_\ell$ is computed by $S_\ell = \arg\max_{S \subseteq [N], |S| \leq K}R(S,\htv)$, and $\htv$ is updated after epoch $\ell$ only when $\mathcal{D}_i^{(\ell)}$ happens for some $i \in [N]$. Let $L$ be the total number of epochs at or before time $T$; we thus have 
$\sum_{\ell=1}^{L} \ind[\mathcal{D}_{i}^{\ell}]  \le \log T$.
 We then have that
 \begin{align*}
 \E[\AS_T] & = \E \sum_{t=1}^{T-1}\ind[S_t\neq S_{t+1}] \\
&  \le \sum_{\ell=1}^{L}\sum_{i=1}^{N}\ind[\mathcal{D}_{i}^{(\ell)}]   = \sum_{i=1}^{N}\sum_{\ell=1}^{L}\ind[\mathcal{D}_{i}^{(\ell)}]\lesssim N\log T.
 \end{align*} 
\end{proof}

\paragraph{The lower bound.} We complement our algorithmic result with the following almost matching lower bound. The theorem states that the number of assortment switches has to be $\Omega(N \log T / \log \log (NT))$, if the algorithm is anytime and incurs only $\sqrt{NT} \times \poly \log (NT)$ regret. The proof of Theorem~\ref{thm:lb-AS-logT} can be found in Appendix~\ref{app:lb-AS-logT}.

\begin{theorem}\label{thm:lb-AS-logT}
There exist universal constants $d_0, d_1 > 0$ such that the following holds. For any constant $C \geq 1$, if an anytime algorithm $\mathcal{A}$ achieves expected regret at most $d_0 \sqrt{NT} (\ln (NT))^{C}$ for all $T$ and all instances with $N$ items, then for any $N \geq 2$, $T_0 \geq N$ and $T_0$ greater than a sufficiently large constant that only depends on $C$, there exists an instance with $N$ items and a time horizon $T \in [T_0, T_0^2]$, such that the expected number of assortment switches before time $T$ is at least $d_1 N \log T / (C \log \log (NT))$.
\end{theorem}
\section{Achieving $O(N \log \log T)$ assortment switch with known horizons} \label{sec:AS-loglogT}

When the time horizon is known to the algorithm, we can exploit this advantage via more carefully designed update policy to achieve only $O(N \log \log T)$ assortment switches. For the convenience of presentation, we first introduce a few notations.

\begin{algorithm2e}[h]
\caption{$\Update(i)$}
\label{alg:update}
$\tau_i \gets \tau_i + 1$;
$T_i^{(\tau_i)}\gets T_i^{(\tau_i-1)}+|\cT(i, \tau_i-1)|$\; \label{step:update-Ti}
$n_i^{(\tau_i)} \gets n_i^{(\tau_i-1)} +n_{i,\tau_i - 1} $; 
% $\hat{n}_{i, \tau_i}\gets 0$; 
% $\cT(i, \tau_i) \gets \emptyset$\;
$\bv_{i, \tau_i} \gets n_i^{(\tau_i)}/T_i^{(\tau_i)}$\;
$\hv_{i, \tau_i}\gets \min\Big\{\hv_{i, \tau_i-1}, 
\bv_{i, \tau_i}+\sqrt{\frac{48\bv_{i, \tau_i}\ln(\sqrt{N}T^2+1)}{T_i^{(\tau_i)}}}
+\frac{48\ln(\sqrt{N}T^2+1)}{T_i^{(\tau_i)}}\Big\}$\;
\end{algorithm2e}

For each item $i \in [N]$, we divide the time periods into consecutive \emph{stages} where the boundaries between any two neighboring stages are marked by the UCB updates for item $i$. Note that the division for the stages may be different for different items. For any $\tau \in \{1, 2, 3, \dots\}$, let $\mathcal{T}(i, \tau)$ be the set of epochs to offer item $i$, in stage $\tau$ for the item. Let $T^{(\tau)}_i = \sum_{\tau' = 1}^{\tau - 1} |\mathcal{T}(i, \tau')|$ be the total number of epochs to offer item $i$, \emph{before} stage $\tau$ for the item, and let $n^{(\tau)}_i$ be the total number of purchases for item $i$ in the epochs counted by $T^{(\tau)}_i$. We can therefore define $\bv_{i, \tau} \defeq n^{(\tau)}_i / T^{(\tau)}_i$ as an unbiased estimate of $v_i$ based on the observations before stage $\tau$. Similarly to \eqref{eq:vi-ucb}, we can define $\hv_{i, \tau}$ as a UCB for $v_i$. The $\Update(i)$ procedure  (formally described in Algorithm~\ref{alg:update}) is invoked whenever the main algorithm decides to conclude the current stage for item $i$ and update the UCB for $v_i$ together with the quantities defined above, where $\tau_i$ is the counter for the number of stages for item $i$, and $n_{i, \tau}$ is the number of purchases observed in stage $\tau$ for item $i$.

The key to the design of our main algorithm for the fixed time horizon setting is a new trigger for updating the UCB values. Let $\tau_0 = \lceil \log \log (T/N) + 1 \rceil$, for each item $i \in [N]$, we will conclude the current stage $\tau_i$ and invoke $\Update(i)$ whenever the following condition $\mathcal{P}(i, \tau_i)$ is satisfied. 
Note that $\mathcal{P}(i, \tau_i)$ is adaptive to the estimated parameters $\hat{v}_{i, \tau_i}$ to customize the number of epochs between assortment switches for each item. More specifically, the smaller $\hat{v}_{i, \tau_i}$ is, the less regret may be incurred by offering item $i$, and therefore the longer we can offer item $i$ without switching and incurring too large regret, and this is reflected in the design of $\mathcal{P}$.

%Note that the time horizon $T$ is involved in the definition of $\mathcal{P}(i, \tau_i)$, and has to be given to the algorithm as input.
%\begin{align}
%\mathcal{P}(i, \tau_i) = \left\{
%\begin{array}{ll}
%|\mathcal{T}(i, \tau_i)| \geq 1 + \sqrt{\frac{T \cdot T_i^{(\tau_i)}}{N }}  &  \text{if}~ \tau_i < \tau_0 \\
%|\mathcal{T}(i, \tau_i)| \geq 1 + \sqrt{\frac{T \cdot T_i^{(\tau_i)}}{N \cdot \hv_{i, \tau_i}}} \text{~and~} \hv_{i, \tau_0} > 1/\sqrt{NT} &  \text{if}~ \tau_i \geq \tau_0 
%\end{array}
%\right. . \label{eq:calP}
%\end{align}
\begin{align}
\mathcal{P}(i, \tau_i) \defeq \left\{
\begin{array}{ll}
|\mathcal{T}(i, \tau_i)| \geq 1 + \sqrt{\frac{T \cdot T_i^{(\tau_i)}}{N }}  &  \text{if}~ \tau_i < \tau_0 \\
|\mathcal{T}(i, \tau_i)| \geq 1 + \sqrt{\frac{T \cdot T_i^{(\tau_i)}}{N \cdot \hv_{i, \tau_i}}}  &\\
\text{~~~~~~~~and~} \hv_{i, \tau_0} > 1/\sqrt{NT} &  \text{if}~ \tau_i \geq \tau_0 
\end{array}
\right. .  \label{eq:calP}
\end{align}
For each epoch $\ell$, we use $\tau_i(\ell)$ to denote the stage (in terms of item $i$) where epoch $\ell$ belongs to. We present the details of our main algorithm in Algorithm~\ref{alg:defer-ucb-loglogT}. The algorithm is terminated whenever the time step $t$ reaches the horizon $T$.

\begin{theorem}\label{thm:AS-UB-loglogT} 
For any given time horizon~$T\geq N^4$, we have the following upper bound for the expected regret:
\[
\expect{\regret_T} \lesssim \sqrt{NT\ln(\sqrt{N}T^2+1)} \cdot \log \log T,
\]
and the following upper bound for the expected number of assortment switches: 
\[
\expect{\AS_T} \lesssim N\log \log T.
\]
\end{theorem}

To prove Theorem~\ref{thm:AS-UB-loglogT}, we first define the desired events. Let 
\begin{align*}
&\event^{(1)}_{i, \tau} \defeq \Big\{\hv_{i, \tau} \geq v_i \text{ and } 
\hv_{i,\tau} \leq v_i +\\
&\qquad \sqrt{\frac{144 v_i\ln(\sqrt{N}T^2+1)}{T_i^{(\tau)}}}
+ \frac{144\ln(\sqrt{N}T^2+1)}{T_i^{(\tau)}}\Big\},
\end{align*}
 and 
 \[
 \event^{(1)} \defeq \cap_{i, \tau} \event^{(1)}_{i, \tau}.\]
 We also let 
\begin{align*}
&\event^{(2)}_{i, \tau} \defeq \Big\{{n}_{i, \tau} \geq \frac{1}{2} v_i|\cT(i, \tau)|, \\
& \qquad  \text{ if } v_i \geq \frac{1}{2}\sqrt{\frac{1}{NT}} \text{ and } |\cT(i, \tau)| \geq \frac{T}{4N\cdot v_i}\Big\},
 \end{align*}
and   \[
\event^{(2)} \defeq \cap_{i, \tau} \event^{(2)}_{i, \tau}.
\]
 Finally, let
 $\event = \event^{(1)} \cap \event^{(2)}$. In Appendix~\ref{app:missing proof of concentration}, we prove the following lemma.

\begin{algorithm2e}[h]
\caption{Deferred Update UCB for Fixed Time Horizon (FH-DUCB)}
\label{alg:defer-ucb-loglogT}
\SetKwInOut{Input}{Input}
\Input{The time horizon $T$.}
Initialize: $\tau_i \gets 1, \hv_{i, \tau_i}\gets 1, n_{i, \tau_i} \gets 0, 
\cT(i, \tau_i) \gets \emptyset,  T_i^{(1)}\gets 0, n_i^{(1)} \gets 0$ 
for all $i\in [N]$\;
$t \gets 0$, $S_0 \gets [N]$\;
% \For{$\ell=1, \dots, \sqrt{\frac{T}{N}}$} {
% 	$\{\Delta_i\}\gets \Exp([N])$\;
% 	\For{$i\in [N]$} {
% 		$\hat{n}_{i, \tau_i}\gets \hat{n}_{i, \tau_i}+\Delta_i$\;
% 		Add $\ell$ to $\cT(i, \tau_i)$.
% 	}
% }
% \For{$i\in [N]$}{\Update($i$).}
% Compute $S_\ell=\mathop{\arg\max}_{S}R(S,\htv_{\ell})$, 
% where $\htv_{\ell} = (v_{i, \tau_i(\ell)})_{i\in[N]}$ and $\tau_i(\ell)$ is the stage such that $\ell \in \cT(i, \tau_i(\ell))$\;
\For{$\ell \gets 1,2,3,\dots,$} {
	$S_{\ell} \gets S_{\ell-1}$\;
	\If{\label{step:inductive}$\exists i: \mathcal{P}(i, \tau_{i})~\text{holds}$}{
	    \Update($i$) for all $i$ such that $\mathcal{P}(i, \tau_{i})$ holds\;
% 		$T_{i^*}^{(\tau_{i^*})}\gets T_{i^*}^{(\tau_{i^*}-1)}+|\cT({i^*}, \tau_{i^*})|$; 
%         $n_{i^*}^{(\tau_{i^*})} \gets n_{i^*}^{(\tau_{i^*}-1)} +\hat{n}_{{i^*}, \tau_{i^*}} $; 
%         $\tau_{i^*} \gets \tau_{i^*} + 1$\;
%         $\hat{n}_{{i^*}, \tau_{i^*}}\gets 0$; 
%         $\cT({i^*}, \tau_{i^*}) \gets \emptyset$\;
%         $\bv_{{i^*}, \tau_{i^*}} \gets n_{i^*}^{(\tau_{i^*})}/T_{i^*}^{(\tau_{i^*})}$\;
%         $\hv_{{i^*}, \tau_{i^*}}\gets \min\left\{\hv_{{i^*}, \tau_{i^*}-1}, 
%         \bv_{{i^*}, \tau_{i^*}}+\sqrt{\frac{48\bv_{i^*}\ln(\sqrt{N}T^2+1)}{T_{i^*}^{(\tau_{i^*})}}}
%         +\frac{48\ln(\sqrt{N}T^2+1)}{T_{i^*}^{(\tau_{i^*})}}\right\}$\;
		Compute $S_\ell\gets \mathop{\arg\max}_{S \subseteq [N] : |S| \leq K}R(S,\htv_{\ell})$
		where $\htv_{\ell} = (\hv_{i, \tau_i(\ell)})_{i\in[N]}$\;
	}
% 	Add $\ell$ to $\cT(i, \tau_i)$ for $i \in [N]$
	$\{\Delta_i\}\gets \Exp(S_{\ell})$\;
	\For{$i\in S$} {
		${n}_{i, \tau_i}\gets {n}_{i, \tau_i}+\Delta_i$;
		Add $\ell$ to $\cT(i, \tau_i)$\;
	}
}
\end{algorithm2e}

\begin{lemma}\label{lem:event prob}
If $T \geq N^4$ and $T$ is greater than a  large enough universal constant, then $\Pr[\event] \geq 1-\frac{14}{T}$.
\end{lemma}

\paragraph{Bounds for the stage lengths.} When $\event$ happens, we can infer the following useful lower bound for the lengths of the stages after  $\tau_0$. The lemma is proved in Appendix~\ref{app:proof-lem-stage-bound-2}.

\begin{lemma}\label{lem:stage-bound-2}
Assume that $T \geq N^4$ and $T$ is greater than a sufficiently large universal constant. Conditioned on $\event^{(1)}$, for each $i \in [N]$, if $\tau_0$ is not the last stage for item $i$, we have that $v_i \geq \frac{1}{2 }\sqrt{\frac{1}{NT}}$. Additionally, if $\hv_{i, \tau_0} > 1/\sqrt{NT}$, then for all $\tau > \tau_0$ such that $\tau$ is not the last stage for $i$, we have that $|\mathcal{T}(i, \tau)|  \geq (T/(2Nv_i))^{1-2^{-\tau+\tau_0+1}}$.
\end{lemma}

\paragraph{Upper bounding the number of assortment switches.} Suppose that there are $L$ epochs before the algorithm terminates. We only need to upper bound $\E \sum_{i=1}^{N} \tau_i(L)$ which upper bounds the number of assortment switches $\E[\AS_T]$. For each $i \in [N]$, if $\tau_i(L) \geq \tau_0$ and $\hv_{i, \tau_0} \leq 1/\sqrt{NT}$, we easily deduce that $\tau_i(L) \leq \tau_0 + 1$ because of the condition $\mathcal{P}(i, \tau_0)$.  Otherwise, assuming that $\hv_{i, \tau_0} > 1/\sqrt{NT}$, by Lemma~\ref{lem:stage-bound-2}, conditioned on $\event^{(1)}$,  we have that $v_i \geq \frac{1}{2 }\sqrt{\frac{1}{NT}}$ and $|\mathcal T(i, \tau)| \geq \frac{T}{4 N v_i}$ for all $\tau \in [\tau_0 + \log \log \frac{T}{2 N v_i} + 1, \tau_i(L) - 1]$. Because of $\event^{(2)}$, we have $n_{i, \tau} \geq \frac{v_i}{2} \cdot |\mathcal T(i, \tau)| \geq \frac{T}{8 N}$ for all $\tau \in [\tau_0 + \log \log \frac{T}{2 N v_i} + 1, \tau_i(L) - 1]$. Therefore, we know that there are no more than $8N$ pairs of $(i, \tau)$ satisfying $\tau \in [\tau_0 + \log \log \frac{T}{2 N v_i} + 1, \tau_i(L) - 1]$. In total, conditioned on $\event$, we have that
\begin{align}
&\quad  \E \sum_{i=1}^{N} \tau_i(L) \nonumber\\
& \lesssim N \tau_0 + \sum_{i = 1}^{N} \ind\Big[\hv_{i, \tau_0} > 1/\sqrt{NT}\Big] \log \log \frac{T}{2 N v_i}  \nonumber\\
&\qquad   + \E \sum_{i=1}^{N} \max\{\tau_i(L) - \tau_0 - \log \log \frac{T}{2 N v_i}, 0\} \nonumber \\
&\lesssim N \log \log T + \sum_{i = 1}^{N} \log \log \frac{T^{3/2}}{N^{1/2} } \lesssim N \log \log T, \label{eq:switch-loglogT}
\end{align}
where the second inequality is because of Lemma~\ref{lem:stage-bound-2}. Finally, since the contribution to the expected number of assortment switches when $\event$ fails  is at most $\Pr[\overline{\event}] \cdot T \leq O(1)$ (because of Lemma~\ref{lem:event prob}), we prove the upper bound for the number of assortment switches in Theorem~\ref{thm:AS-UB-loglogT}.

\paragraph{Upper bounding the expected regret.} Let $E^{(\ell)}$ be the length of epoch $\ell$, i.e., the number of time steps taken in epoch $\ell$. Note that $E^{(\ell)}$ is a geometric random variable with mean value $(1+\sum_{i\in S_\ell}v_i)$. Also recall that there are $L$ epochs in total.
Letting $S^*$ be the optimal assortment, conditioned on event $\event^{(1)}$, we have that
\begin{align}
 \expect{\regret_T} 
&=\E\sum_{\ell=1}^{L}E^{(\ell)}(R(\optS,\bm{v})-R(S_\ell,\bm{v})) \nonumber \\
&=\E\sum_{\ell=1}^{L}\left(1+\sum_{i\in S_\ell}v_i\right)(R(\optS,\bm{v})-R(S_\ell,\bm{v})) \nonumber \\
&\leq \E \sum_{\ell=1}^{L} \sum_{i\in S_\ell}(\hv_{i, \tau_i(\ell)}-v_i)\nonumber\\
&= \E \sum_{i=1}^N \sum_{\ell: i\in S_{\ell}} (\hv_{i, \tau_i(\ell)}-v_i) \nonumber\\
&= \E \sum_{i=1}^N \sum_{\tau=1}^{\tau_i(L)} \sum_{\ell \in \cT(i, \tau)} (\hv_{i, \tau}-v_i) ,\label{eq:regret-loglogT}
\end{align}
where the inequality is due to Lemma~\ref{lem:bound-regret-by-ucb-value}. In the next lemma, we upper bound the contribution from each item $i$ and stage $\tau$ to the upper bound in \eqref{eq:regret-loglogT}. The lemma is proved in Appendix~\ref{app:proof-lem-regret-in-stage}.

\begin{lemma}\label{lem:regret in stage}
Conditioned on event $\event^{(1)}$, for any item $i$ and any stage $\tau \leq \tau_i(L)$, we have that
\[
\sum_{\ell \in \cT(i, \tau)} (\hv_{i, \tau}-v_i) 
\lesssim \sqrt{{T\ln(\sqrt{N}T^2+1)}/{N}}.
\]
\end{lemma}

Combining Lemma~\ref{lem:event prob}, Lemma~\ref{lem:regret in stage}, inequalities \eqref{eq:switch-loglogT} and \eqref{eq:regret-loglogT}, we have that
\begin{align*}
& \expect{\regret_T}  \leq T \cdot \Pr[\overline{\event^{(1)}}] +  \expect{\regret_T \given \event^{(1)}}  \\
  \lesssim & 1 + \E \sum_{i=1}^{N} \tau_i(L) \times   \sqrt{\frac{T\ln(\sqrt{N}T^2+1)}{N}} \\
 \lesssim & \sqrt{NT\ln(\sqrt{N}T^2+1)} \cdot \log \log T,
\end{align*}
proving  the expected regret upper bound in Theorem~\ref{thm:AS-UB-loglogT}.

\paragraph{The lower bound.} We prove the following matching lower bound in Appendix~\ref{app:lb-AS-loglogT}.

\begin{theorem}\label{thm:lb-AS-loglogT}
For any constant $C \geq 0$ and time horizon $T$, if an algorithm $\mathcal{A}$ achieves expected regret $\E[\regret_T]$ at most $\frac{1}{7525} \cdot \sqrt{NT} (\ln (NT))^C$ for all $N$-item instances, then there exists an $N$-item instance  such that the expected number of assortment switches is 
\[
\E[\AS_T] = \Omega(N \log \log T).
\]
%, where  the constant  hidden in the $\Omega(\cdot)$ notation only depends on C.
\end{theorem}

\section{Optimizing the number of item switches} \label{sec:IS}

In this section, we study how to minimize the item switch cost while still achieving $\tilde{O}(\sqrt{NT})$ regret.

\begin{algorithm2e}[h]
	\caption{The Exponential Stride UCB algorithm (\Talgo) for MNL-Bandit}
	\label{alg:trisection}
		Initialize: $\htheta \gets 1,\epsilon_1 \gets 1/3$, $c_1 \gets 44840$\;
		\For{$\tau\gets 1,2,3,\dots$} {
			$\tmax\gets c_1 N\ln^3(NT/\delta)/\epsilon_\tau^2$\;\label{line:tmax}
			\textbf{if} $\chk(\htheta-3\epsilon_\tau, \htheta-\epsilon_\tau, \tmax)$ \textbf{then} $\htheta\gets\htheta-\epsilon_\tau$\; \label{line:first-invocation}
%			$b\gets \chk(\htheta-3\epsilon_\tau, \htheta-\epsilon_\tau, \tmax)$\;\label{line:first-invocation}
%			\If{$b=\true$}{
%				$\htheta\gets\htheta-\epsilon_\tau$\;
%			}
			$\epsilon_{\tau+1}\gets \frac{2}{3}\epsilon_{\tau}$\;
		}
\end{algorithm2e}

We now propose a new algorithm, Exponential Stride UCB (ESUCB), to achieve an item switching cost that is linear with $N$ and poly-logarithmic with $T$. The specific guarantee of the ESUCB algorithm is presented in Theorem~\ref{thm:ESUCB}, the main theorem of this section. The key idea of the algorithm is to decouple the learning of the optimal expected revenue and the optimal assortment, which is made possible by the following lemma.
\begin{lemma}\label{lem:characterization-of-G}
Define $G(\theta)\defeq R(S_\theta,\bm{v}),$ where $S_\theta\defeq\arg\max_{S\subseteq [N]: |S| \leq K}\left(\sum_{i\in S}v_i(r_i-\theta)\right)$. There exists a unique $\optt$ such that 
\[
G(\optt)=\optt=\max_{|S| \leq K} R(S,\bm{v}).
\]
Moreover, 
\begin{itemize}
\item[(1)] for any $\theta<\optt$, we have that $G(\theta)>\theta$, and 
\item[(2)] for any $\theta>\optt$, we have that $G(\theta)<\theta$.
\end{itemize}
\end{lemma}

The proof of Lemma~\ref{lem:characterization-of-G} is deferred to Appendix~\ref{app:proof-lem-char-G}. Motivated by the lemma, we present our ESUCB algorithm in Algorithm~\ref{alg:trisection}. The algorithm learns the optimal revenue $\optt$ in the main loop, using a sequence of exponentially decreasing learning step size $\epsilon_\tau$. For each estimate $\htheta$, the \chk procedure (Algorithm~\ref{alg:check}) learns the assortment $S_{\htheta}$ via the UCB method with deferred updates. (More precisely speaking, the algorithm learns $S_{\htheta-\epsilon_\tau}$ and $S_{\htheta-3\epsilon_\tau}$, and at Line~\ref{line:if}, chooses one of them based on the UCB estimation $\hat{\rho}$ for the expected revenue of $S_{\htheta-\epsilon_\tau}$.)  In the \chk procedure, the variable $t$ keeps the count of time steps and is updated in \Exp. We also make the following notes: 1) The ESUCB algorithm needs the horizon $T$ as input, and uses a confidence parameter $\delta$, which is usually set as $1/T$. The whole algorithm terminates whenever the horizon $T$ is reached. 2) At the optimization steps (Lines \ref{line:erm-with-fixed-theta-0} and \ref{line:erm-with-fixed-theta} of Algorithm~\ref{alg:check}), we have to adopt a deterministic tie breaking rule, e.g., we let the $\arg\max$ operator to return the $S$ such that $\sum_{i \in S} 2^i$ is minimized among multiple maximizers.

\begin{theorem}\label{thm:ESUCB}
%We set $c_2, c_3 > 0$ to be large enough constants, and $c_1 > 0$ to be a large enough constant given $c_2$ and $c_3$ (for example, $c_1=44840=2(c_2+c_3)$,$c_2=688$, and $c_3=21732$, as used in the proof of this theorem). Then, for $\delta = 1/T$,

Setting $\delta = 1/T$, we have the following upper bound for the expected regret of ESUCB:
\[
 \expect{\regret_T} \lesssim \sqrt{NT} \cdot \log^{1.5} (N T),
 \]
and the item switching cost for ESUCB is 
\[
\expect{\IS_T} \lesssim N \log^2  T.
\]
\end{theorem}
To prove Theorem~\ref{thm:ESUCB}, we upper bound the item switching cost and the expected regret separately.

\paragraph{Upper bounding the item switch cost.}  Since the estimate of $\optt$ is fixed in \chk, the outcome of $\mathop{\arg\max}_{S: |S| \leq K} \sum_{i\in S}\hv_i(r_i-\theta)$ (corresponding to Lines \ref{line:erm-with-fixed-theta-0} and \ref{line:erm-with-fixed-theta} of Algorithm~\ref{alg:check}) becomes more stable compared to that of $\mathop{\arg\max}_{S: |S| \leq K} R(S, \bm{\hv})$ in previous algorithms. Exploiting this advantage, we upper bound the number of item switches incurred by each call of \chk as follows. The lemma is proved in Appendix~\ref{app:proof-lem-item-switch-check}. 

\begin{algorithm2e}[h]
	\caption{$\chk(\theta_l,\theta_r, \tmax)$}
	\label{alg:check}
		Initialize: $\hv_i\gets 1, T_i\gets 0, n_i\gets 0$ for all $i\in [N]$, $c_2 \gets 688$, $c_3 \gets 21732$\;
		$\rho \gets 0$, $\hat{\rho} \gets 1$, $ b\gets \false$, $t\gets 0$\;
		\For{$\ell \gets 1,2,3,\dots$} {
			\If{$\hat{\rho}<\theta_r$}{\label{line:if}
				$b\gets\true$\;
				{\small $S_\ell\gets \mathop{\arg\max}_{S \subseteq [N], |S| \leq K}\left(\sum_{i\in S}\hv_i(r_i-\theta_l)\right)$}\; \label{line:erm-with-fixed-theta-0}
				$\{\Delta_i\}\gets \Exp(S_\ell)$\;
			}
			\Else{
				{\small $S_\ell\gets\mathop{\arg\max}_{S \subseteq [N], |S| \leq K}\left(\sum_{i\in S}\hv_i(r_i-\theta_r)\right)$}\;\label{line:erm-with-fixed-theta}
				$\{\Delta_i\}\gets \Exp(S_\ell)$\;
				$\rho\gets \rho+\sum_{i \in S_{\ell}} \Delta_i \cdot r_i$;
				$\hat{\rho}\gets \frac{1}{t}\big(\rho + c_2\sqrt{N\tmax \ln^3(NT/\delta)}+c_3N\ln^3(NT/\delta)\big)$\;
			}
			\textbf{if} $t\ge \tmax$ \textbf{then} \textbf{return} $b$\;
%			\If{$t\ge \tmax$}{
%				\Return $b$\;
%			}
			\For{$i\in S_\ell$} {
				$n_i\gets n_i+\Delta_i$, $T_i\gets T_i+1$\;
				\If{$T_i=2^{k}$ for some $k\in \Z$}{
					$\bv_i\gets n_i/T_i$;
					$\hv_i\gets \min\big\{\hv_i,\bv_i+\sqrt{\frac{196\bv_i\log(NT/\delta+1)}{T_i}}+\frac{292\log(NT/\delta+1)}{T_i}\big\}$\;
				}
			}
		}
\end{algorithm2e}

\begin{lemma}\label{lem:item-switch-check}
The item switch cost incurred by any invocation  $\chk(\theta_l, \theta_r ,\tmax)$ is $O(N \log T)$.
\end{lemma}

Since the $\tau$ loop in Algorithm~\ref{alg:trisection} iterates for only $O(\log T)$ times, Lemma~\ref{lem:item-switch-check} easily implies an $O(N \log^2 T)$ item switching cost upper bound for ESUCB. We also note that this bound can be improved to $O(N \log T)$ via a slight modification to the algorithm which is elaborated in Appendix~\ref{app:ESUCB-further-improve}.

\paragraph{Upper bounding the expected regret.}  We first provide the following guarantees for  \chk. 

\begin{lemma}[Main Lemma for \chk]\label{lem:check-main}
For any invocation $\chk(\theta_l,\theta_r,\tmax)$, with probability at least $(1-\delta/T)$, the following statements hold.
\begin{itemize}
\item[(a)] If \chk returns $\true$, then $G(\theta_r)<\theta_r$.
\item[(b)] If \chk returns $\false$, then 
\[
\theta^\star\ge \theta_r-\frac{2}{\tmax}\left(c_2\sqrt{N\tmax\ln^3\frac{NT}{\delta}}+c_3N\ln^3\frac{NT}{\delta}\right).
\]
\item[(c)] Let $r_{\chk}^{(t)}$ be the reward at time step $t$ in this invocation. If $\theta_l\le \theta^\star$, then we have that 
\begin{align*}
&\tmax\theta_l-\E\left[\sum_{t=1}^{\tmax}r_{\chk}^{(t)}\right] \\
& \qquad \lesssim \sqrt{N\tmax\ln^3 (NT/\delta)}+N\ln^3(NT/\delta).
\end{align*}
\end{itemize}
\end{lemma}
Proof of Lemma~\ref{lem:check-main} is built upon Lemma~\ref{lem:characterization-of-G} and deferred to Appendix~\ref{app:proof-lem-check-main}. 

Let $\calQ_\tau$ be the event that the statements $(a)-(c)$ hold for the invocation of $\chk$ at iteration $\tau$ of Algorithm~\ref{alg:trisection}, and let $\calQ$ be the event that $\calQ_\tau$ holds every all $\tau$. By Lemma~\ref{lem:check-main} and a union bound, we immediately have that $\Pr[\calQ]\ge 1-\delta$. The next lemma, built upon Lemma~\ref{lem:characterization-of-G} and Lemma~\ref{lem:check-main}, shows that $\hat{\theta}$ in Algorithm~\ref{alg:trisection} is always an upper confidence bound for the true parameter $\optt$, and converges to $\optt$ with a decent rate.
\begin{lemma}\label{lem:ucb-theta}
Let $\hat{\theta}^{(\tau)}$ be the value of $\hat{\theta}$ at the beginning of iteration $\tau$ of Algorithm~\ref{alg:trisection}. Conditioned on event $\calQ$, for any iteration $\tau=1,2,3,\dots$, we have that $\hat{\theta}^{(\tau)}-3\epsilon_\tau\le \optt\le \hat{\theta}^{(\tau)}$.
\end{lemma}
\begin{proof}

Recall that for every $\tau = 1, 2, 3, \dots$, we need to prove
\begin{equation}\textstyle \hat{\theta}^{(\tau)}-3\epsilon_\tau\le \optt\le \hat{\theta}^{(\tau)}.\label{equ:ucb-theta}\end{equation}

We prove this by induction. For iteration $\tau=1$, \eqref{equ:ucb-theta} trivially holds since $0\le r_i\le 1$ and therefore $0 \leq \optt \leq 1$.

Now suppose \eqref{equ:ucb-theta} holds for iteration $\tau$, we will establish \eqref{equ:ucb-theta} for iteration $(\tau +1)$. Consider the invocation of $\chk(\theta_l,\theta_r,\tmax)$ at iteration $\tau$, where $\theta_l = \htheta^{(\tau)} - 3\epsilon_\tau$ and $\theta_r = \htheta^{(\tau)} - \epsilon_\tau$. We discuss the following two cases.

\noindent \underline{\it Case 1.}\ When the \chk procedure returns \true, by Lemma~\ref{lem:check-main} we have that $G(\theta_r)<\theta_r.$  By Lemma~\ref{lem:characterization-of-G}, we have that $\theta_r> \optt.$ Therefore, by Line~\ref{line:first-invocation} and the induction hypothesis we have that $\hat{\theta}^{(\tau+1)}=\hat{\theta}^{(\tau)} - \epsilon_\tau = \theta_r>\optt,$ and $\hat{\theta}^{(\tau+1)}-3\epsilon_{\tau+1}=\theta_r-2\epsilon_{\tau}=\hat{\theta}^{(\tau)}-3\epsilon_{\tau} \leq \optt$, proving \eqref{equ:ucb-theta}.

\noindent \underline{\it Case 2.}\  When the \chk procedure returns \false, by Lemma~\ref{lem:check-main}, we have that  
{\small $$\optt\ge \theta_r-\frac{1}{\tmax}\left((c_2+8)\sqrt{N\tmax\ln^3\frac{NT}{\delta}}+c_3N\ln^3\frac{NT}{\delta}\right).$$}Recall that at Line~\ref{line:tmax} we set $\tmax=c_1N\ln^3(NT/\delta)/\epsilon_\tau^2$. For large enough $c_1$, this implies that 
\[
\optt\ge \theta_r-\epsilon_\tau=\hat{\theta}^{(\tau)}-2\epsilon_\tau=\hat{\theta}^{(\tau+1)}-3\epsilon_{\tau+1}.
\]
 By Line~\ref{line:first-invocation}  and the induction hypothesis we have that 
$ \hat{\theta}^{(\tau+1)}=\hat{\theta}^{(\tau)}\ge \optt$,
  finishing the proof of \eqref{equ:ucb-theta}.
\end{proof}

Finally we upper bound  the expected regret of Algorithm~\ref{alg:trisection}.
\begin{lemma}\label{lem:trisection-regret}
With probability at least $1-\delta$, the expected regret incurred by Algorithm~\ref{alg:trisection} is $O(\sqrt{NT} \log^{1.5} (NT/\delta))$. Therefore, if we set $\delta = 1/T$, we have that 
\[
\E[\regret_T] \lesssim \sqrt{NT} \log^{1.5} (NT).
\]
\end{lemma}

\begin{proof}
Throughout the proof we condition on the event $\calQ$, which happens probability at least $(1-\delta)$. We first prove that at iteration $\tau$ of Algorithm~\ref{alg:trisection}, the expected regret for this iteration is bounded by $\tilde{O}(N/\epsilon_\tau).$ Consider the invocation $\chk(\theta_l,\theta_r,\tmax)$ at Line~\ref{line:first-invocation}. Recall that we define $\tmax=c_1N\ln^3(NT/\delta)/\epsilon_\tau^2.$ Combining with statement (c) of Lemma~\ref{lem:check-main} and Lemma~\ref{lem:ucb-theta}, the expected regret of this invocation is bounded by (where the $O(N)$ term is due to the last epoch that might run over time $\tmax$),
\begin{align}
& \E\left[\optt\cdot \tmax-\sum_{t=1}^{\tmax}r_{\chk}^{(t)}\right] + O(N) \nonumber\\
 \lesssim  \; & \tmax(\optt-\theta_l)  + \E\left[\theta_l\cdot \tmax-\sum_{t=1}^{\tmax}r_{\chk}^{(t)}\right] + O(N) \nonumber \\
\lesssim \; & \tmax(\optt-\theta_l)+N\ln^3(NT/\delta)/\epsilon_\tau. \label{eq:trisection-regret-1}
\end{align}
By Lemma~\ref{lem:ucb-theta}, we have that $\optt-\theta_l\lesssim \epsilon_\tau$. Therefore, \eqref{eq:trisection-regret-1} is upper bounded by $O(N\ln^3(NT/\delta)/\epsilon_\tau)$.

Since $\chk(\theta_l,\theta_r,\tmax)$ runs for at least $\tmax$ time steps, the second to the last iteration $(\tau_{\rm max} - 1)$ satisfies that $c_1N\ln^3(NT/\delta)/\epsilon_{\tau_{\rm max}-1}^2\le T$, which means that 
\[
\epsilon_{\tau_{\rm max}}\gtrsim \sqrt{N\log^3(NT/\delta)/T}.
\]
Since $\epsilon_\tau$ is an exponential sequence, the overall expected regret is bounded by the order of
\[
\sum_{\tau=1}^{\tau_{\rm max}}N\log^3(NT/\delta)/\epsilon_\tau\lesssim \sqrt{NT\log^3(NT/\delta)}.
\]
\end{proof}

\paragraph{Refined and non-trivial item switching cost upper bound for the AT-DUCB algorithm.} Since an assortment switch may incur at most $2K$ item switches, Theorem~\ref{thm:ucb-doubling-regret} trivially implies that Algorithm~\ref{alg:doubling-ucb} (AT-DUCB) incurs at most $O(KN \log T)$ item switches, which is upper bounded by $O(N^2\log T)$ since $K = O(N)$. 

In Appendix~\ref{app:proof-of-theorem-FH-DUCB}, we present a refined analysis showing that the item switching cost of AT-DUCB is at most $O(N^{1.5} \log T)$. While it is not clear to us whether  the dependence on $N$ delivered by this analysis is optimal, we also discuss the relationship between the analysis and an extensively studied (but not yet fully resolved) geometry problem, namely the maximum number of planar $K$-sets. We hope that further study of this relationship might lead to improvement of both upper and lower bounds of the item switching cost of AT-DUCB. Please refer to Appendix~\ref{app:proof-of-theorem-FH-DUCB} for more details.

\section{Conclusion}
In this paper, we present algorithms for MNL-bandits that achieve both almost optimal regret and assortment switching cost, in both anytime and fixed-horizon settings. We also design the ESUCB algorithm that achieves the almost optimal regret and item switching cost $O(N \log^2 T)$. For future directions, it is interesting to study whether it is possible to achieve an item switching cost of $O(N \log T)$ in the anytime setting and $O(N \log \log T)$ in the fixed-horizon setting. Also, as mentioned in Section~\ref{sec:IS} (and Appendix~\ref{app:proof-of-theorem-FH-DUCB}), given the simplicity of our AT-DUCB algorithm, it is worthwhile to further refine the bounds for its item switching cost.

\section*{Acknowledgement}
Part of the work done while Kefan Dong was a visiting student at UIUC. Kefan Dong and Yuan Zhou were supported in part by a Ye Grant and a JPMorgan Chase AI Research Faculty Research Award. Qin Zhang was supported in part by NSF IIS-1633215, CCF-1844234 and CCF-2006591.

\bibliography{ref}
\bibliographystyle{icml2020}

\newpage

\onecolumn

\appendix

{\bf \LARGE \centering Appendix \par}

\section{Proof of the regret upper bound in Theorem~\ref{thm:ucb-doubling-regret}} \label{app:proof-of-theorem-AT-DUCB}

In this section we complete the proof of Theorem~\ref{thm:ucb-doubling-regret} for completeness. The proof is almost identical to that in \cite{AAGZ17} except for the handling of the deferred UCB value updates. 

The following lemma proves that $\hat{v}_i$ is indeed an upper confidence bound of true parameter $v_i$ with high probability, and converges to the true value with decent rate.
\begin{lemma}[Lemma 4.1 of \cite{AAGZ17}]\label{lem:ucb-value-cite}
For any $\ell=1,2,3, \dots$, in Algorithm~\ref{alg:doubling-ucb}, at Line~\ref{line:doubling-if} immediately after the $\ell$-th epoch, the following two statements hold,
\begin{itemize}
	\item[1.] With probability at least $1-\frac{6}{N\ell}$, $\frac{n_i}{T_i} + \sqrt{\frac{48 (n_i/T_i)\ln(\sqrt{N}\ell+1)}{T_i}}+\frac{48\ln(\sqrt{N}\ell+1)}{T_i}\ge v_i$ for any $i\in [N],$
	\item[2.] With probability at least $1-\frac{7}{N\ell}$, for any $i\in [N],$
	$$\frac{n_i}{T_i} + \sqrt{\frac{48 (n_i/T_i)\ln(\sqrt{N}\ell+1)}{T_i}}+\frac{48\ln(\sqrt{N}\ell+1)}{T_i} -v_i \leq \sqrt{\frac{144 v_i\ln(\sqrt{N}\ell+1)}{T_i}}+\frac{144 \ln(\sqrt{N}\ell+1)}{T_i}.$$
\end{itemize}
\end{lemma} 

By the update rule, Lemma~\ref{lem:ucb-value} can be extended to $\{\hv_i\}$ as follows.
\begin{lemma}\label{lem:ucb-value}
For any $\ell=1,2,3,\cdots$,  the following two statements hold at the end of the $\ell$-th iteration of the outer for-loop of Algorithm~\ref{alg:doubling-ucb}.
\begin{itemize}
	\item[1.] With probability at least $1-\frac{6}{N\ell}$, $\hat{v}_{i}\ge v_i$ for any $i\in [N],$
	\item[2.] With probability at least $1-\frac{7}{N\ell}$, for any $i\in [N],$
	$$\hat{v}_{i}  -v_i\lesssim \sqrt{\frac{v_i\log(\sqrt{N}\ell+1)}{T_i}}+\frac{\log(\sqrt{N}\ell+1)}{T_i}.$$
\end{itemize}
\end{lemma} 
\begin{proof}
For any epoch $\ell$, let $T_i'$ and $\hat{v}_i'$ be the value of $T_i$ and $\hat{v}_i$ at the last update. Then we have, $\hat{v}_i=\hat{v}_i'$ and $T_i'\le 2T_i$. Inherited from Lemma~\ref{lem:ucb-value-cite}, we have $\hat{v}_i= \hat{v}'_{i}\ge v_i$. And 
$$\hat{v}_{i}-v_i= \hat{v}'_i-v_i\lesssim \sqrt{\frac{v_i\log(\sqrt{N}\ell+1)}{T'_i}}+\frac{\log(\sqrt{N}\ell+1)}{T'_i}\lesssim \sqrt{\frac{v_i\log(\sqrt{N}\ell+1)}{T_i}}+\frac{\log(\sqrt{N}\ell+1)}{T_i}.$$
\end{proof}

Once we establish Lemma~\ref{lem:ucb-value}, the proof of the regret upper bound in Theorem~\ref{thm:ucb-doubling-regret} is identical to that in \cite{AAGZ17}. We include the proof here for completeness.

The next lemma shows that the expect regret for one epoch is bounded by the summation of estimation errors in the assortment.
\begin{lemma}[Lemma A.4 of \cite{AAGZ17}]\label{lem:bound-regret-by-ucb-value} For any epoch $\ell$, if $r_i\in[0,1]$ and $0\le v_i\le \hat{v}_{i}$ hold for every $i \in [N]$ at the beginning of the $\ell$-th iteration of the outer for-loop in Algorithm~\ref{alg:doubling-ucb}, we have that
$$\left(1+\sum_{i\in S_{\ell}}v_i\right)(R(S_{\ell}, \htv)-R(S_\ell, \bm{v}))\le \sum_{i\in S_\ell}(\hat{v}_i-v_i).$$
\end{lemma}
As a corollary, we have the following lemma, which is an analog to Lemma 4.3 of \cite{AAGZ17}.
\begin{lemma}\label{lem:ucb-value-combined} Given that $r_i\in[0,1]$ for every $i \in [N]$, for any epoch $\ell=1,2,3,\dots$, with probability at least $\frac{13}{\ell}$ we have that
$$\left(1+\sum_{i\in S_{\ell}}v_i\right)(R(S_{\ell}, \htv)-R(S_\ell, \bm{v}))\lesssim \sqrt{\frac{v_i\log(\sqrt{N}\ell+1)}{T_i}}+\frac{\log(\sqrt{N}\ell+1)}{T_i}.$$
\end{lemma}
\begin{proof}
Combine Lemma~\ref{lem:ucb-value} and Lemma~\ref{lem:bound-regret-by-ucb-value}.
\end{proof}

We will also use the following lemma which is proved in \cite{AAGZ17}.
\begin{lemma}[Lemma A.3 of \cite{AAGZ17}]\label{lem:OPTR-monotone-in-v}
If $v_i \leq \hat{v_i}$ holds for every $i \in [N]$, then we have that $R(\optS, \htv) \geq R(\optS, \bm{v})$. 
\end{lemma}

Now we complete the proof of Theorem~\ref{thm:ucb-doubling-regret}.
\begin{proof}[Proof of the regret upper bound in Theorem~\ref{thm:ucb-doubling-regret}]
Let $E^{(\ell)}$ be the length of epoch $\ell$. That is, the number of time steps taken in epoch $\ell$. Note that $E^{(\ell)}$ is a geometric random variable with mean $(1+\sum_{i\in S_\ell}v_i).$ As a result,
\begin{align*}
\E[\regret_T] & =\E\left[\sum_{\ell=1}^{L}E^{(\ell)}(R(\optS,\bm{v})-R(S_\ell,\bm{v}))\right] \\
& \leq \E\left[\sum_{\ell=1}^{L}E^{(\ell)}\left(R(\optS,\htv)-R(S_\ell,\bm{v}) + \frac{6}{ \ell}\right)\right] \\
& \leq \E\left[\sum_{\ell=1}^{L}E^{(\ell)}\left(R(S_{\ell}, \htv)-R(S_\ell,\bm{v}) + \frac{6}{\ell}\right)\right] \\
&=\E\left[\sum_{\ell=1}^{L}\left(1+\sum_{i\in S_\ell}v_i\right)\left(R(\optS,\htv)-R(S_\ell,\htv) +  \frac{6}{\ell}\right)\right],
\end{align*}
where the first inequality is due to Lemma~\ref{lem:OPTR-monotone-in-v} and Lemma~\ref{lem:ucb-value}. 
Let $\Delta R^{(\ell)}\defeq\left(1+\sum_{i\in S_\ell}v_i\right)(R(\optS,\htv)-R(S_\ell,\htv) + 6/\ell)$ for shorthand. We use $T_i^{(\ell)}$ to denote the value of variable $T_i$ at the beginning of epoch $\ell$. By Lemma~\ref{lem:ucb-value-combined}, we have
\begin{align*}
\E[\Delta R^{(\ell)}]&\lesssim \frac{1}{\ell}\left(1+\sum_{i\in S_\ell}v_i\right)+\E\left[\sum_{i\in S_\ell}\left(\sqrt{\frac{v_i\log (\sqrt{N}T+1)}{T_i^{(\ell)}}}+\frac{\log (\sqrt{N}T+1)}{T_i^{(\ell)}}\right)\right].
\end{align*}
As a consequence,
\begin{align}
\E[\regret_T] \lesssim &\;\sum_{\ell=1}^{L}\left(\frac{1}{\ell}\left(1+\sum_{i\in S_\ell}v_i\right)+\E\left[\sum_{i\in S_\ell}\left(\sqrt{\frac{v_i\log (\sqrt{N}T+1)}{T_i^{(\ell)}}}+\frac{\log (\sqrt{N}T+1)}{T_i^{(\ell)}}\right)\right]\right) \nonumber \\
\lesssim &\; N\log T+\sum_{\ell=1}^{L}\E\left[\sum_{i\in S_\ell}\left(\sqrt{\frac{v_i\log (\sqrt{N}T+1)}{T_i^{(\ell)}}}+\frac{\log (\sqrt{N}T+1)}{T_i^{(\ell)}}\right)\right] \nonumber \\
\lesssim &\; N\log T+\E\left[N\log^2 (\sqrt{N}T+1)+\sum_{i\in [N]}\sqrt{v_iT_i^{(L)}\log (\sqrt{N}T+1)}\right] \nonumber\\
\lesssim &\; N\log^2 (\sqrt{N}T+1)+\sum_{i\in [N]}\sqrt{\E[v_iT_i^{(L)}]\log (\sqrt{N}T+1)}. \label{eq:proof-ucb-doubling-regret}
\end{align}

Note that $\E[E_\ell]=1+\sum_{i\in S_\ell}v_i.$ We have
$$\sum_{i\in [N]}v_iT_i^{(L)}=\sum_{\ell=1}^{L}\sum_{i\in S_\ell}v_i\le \sum_{\ell=1}^{L}\E[E_\ell]\le T.$$
As a result, by Jensen's inequality we get that
\begin{align*}
\eqref{eq:proof-ucb-doubling-regret} \lesssim N\log^2 (\sqrt{N}T+1)+\sqrt{NT\log (\sqrt{N}T+1)},
\end{align*}
which concludes the proof.
\end{proof}
\section{Ommitted proofs for the FH-DUCB algorithm in Section~\ref{sec:AS-loglogT} }

\subsection{Proof of Lemma \ref{lem:event prob}} \label{app:missing proof of concentration}

By Lemma \ref{lem:ucb-value-cite}, we have that 
$\Pr[\neg\event^{(1)}_{i, \tau_i}] \leq \frac{13}{NT^2}$. 
Via a union bound, we have that
$$\Pr[\neg\event^{(1)}] \leq \sum_{i, \tau_i} \Pr[\neg\event^{(1)}_{i, \tau_i}] \leq \frac{13}{T}.$$ 
Next we introduce the following concentration inequality for geometric random variables. 
\begin{lemma}[Theorem 1 and Proposition 1 of \cite{jin2019asymptotically}]\label{lem:geometric concentration}
For any $m$ i.i.d.\ geometric random variables $x_1, \dots, x_m$ with parameter $p$, 
i.e., $\Pr[x_i = k] = p(1-p)^k$, 
we have 
\begin{align*}
\Pr\left[\sum_{i=1}^m x_i < \frac{m(1-p)}{2p}\right] 
\leq \exp\left(-m \cdot \frac{1-p}{8}\right) .
\end{align*}
\end{lemma}
% By the inductive formula in Step \ref{step:inductive}, 
% since by definition we have $\hv_{i, \tau_i} \leq 1$, 
% we have 
% \begin{align*}
% |\cT(i, \tau_i)| \geq \sqrt{\frac{T\cdot T_{i}}{N\cdot \hv_{i, \tau_i}}}
% \geq \sqrt{\frac{T\cdot T_{i, \tau_{i-1}}}{N}}
% \geq (T/N)^{1-2^{-\tau_i}}. 
% \end{align*}

% Therefore, after $\log\log_2 \frac{T}{N}$ switches cause by item $i$, 
% the number of periods that offers $i$ is at least $\frac{T}{N}$, 
% i.e., $T_i \geq \frac{T}{N}$ when $\tau_i \geq \log\log_2 \frac{T}{N}$. 

Note that ${n}_{i, \tau_i}$ is the sum of $|\cT(i, \tau_i)|$ independent geometric random variables with parameter $p = \frac{1}{1+v_i}$ (by Observation~\ref{ob:unbiased-est}). 
Substituting $v_i \geq \frac{1}{2}\sqrt{\frac{1}{NT}}$
and $m = |\cT(i, \tau_i)| \geq \frac{T}{4N v_i}$,
we have $\frac{(1-p)}{2p} = \frac{v_i}{2}$ and
\begin{align*}
\Pr\left[{n}_{i, \tau_i} < \frac{1}{2}v_i \cdot |\cT(i, \tau_i)|\right] 
&\leq \exp\left(-|\cT(i, \tau_i)| \cdot \frac{1-p}{8}\right) \\
&\leq \exp\left(-\frac{T}{4N v_i} \cdot \frac{1-\frac{1}{1+v_i}}{8}\right) \\
&\leq \exp\left(-\frac{T}{64N}\right) 
% \\
% &\leq \exp(-\frac{1}{48} \cdot T^{1/8}) 
\leq \frac{1}{NT^2},
\end{align*}
where the last inequality holds for $T$ such that $T \geq N^4$ and $T$ greater than a sufficiently large universal constant. 
By a union bound, we have that
\begin{align*}
\Pr[\neg\event^{(2)}] \leq \frac{1}{T}. 
\end{align*}
Therefore, we have that
\[
\Pr[\event] \geq 1 - \Pr[\neg\event^{(1)}] + \Pr[\neg\event^{(2)}] \geq 1-\frac{14}{T},
\]
proving the lemma. 

\subsection{Proof of Lemma~\ref{lem:stage-bound-2}}  \label{app:proof-lem-stage-bound-2}

We first state the following lemma, showing that for any item and before stage $\tau_0$, the stage lengths quickly grows to $T/N$.
\begin{lemma}\label{lem:stage-bound-1}
For each $i \in [N]$ and $\tau \leq \tau_0$, if $\tau$ is not the last stage for $i$, it holds that $|\cT(i, \tau)|  \geq (T/N)^{1-2^{-\tau+1}}$.
\end{lemma}
Lemma~\ref{lem:stage-bound-1} can be proved by combining the condition $\mathcal{P}(i, \tau)$ for $\tau < \tau_0$ and $\tau = \tau_0$ (also noting that $\hv_{i, \tau} \leq 1$ for all $\tau$) and the following fact (whose proof is via straightforward induction and omitted).
\begin{fact}\label{fact:seq}
For $M \geq 0$ and a sequence $a_0, a_1, a_2, \dots$ such that $a_i \geq 1 + \sqrt{M a_{i -1}}$ for all $i \geq 1$, we have that $a_\tau \geq M^{1 - 2^{-\tau+1}}$ for all $\tau \geq 1$.
\end{fact}

Now we are ready to prove Lemma~\ref{lem:stage-bound-2}.
\begin{proof}[Proof of Lemma~\ref{lem:stage-bound-2}]
We have that $|\mathcal{T}(i, \tau_0)| \geq \frac{T}{2N}$ because of Lemma~\ref{lem:stage-bound-1}. %Since $\hv_{i, \tau} \leq 1$ for all $\tau$, by the condition $\mathcal{P}(i, \tau)$ when $\tau > \tau_0$, we also have $|\mathcal{T}(i, \tau)|  \geq \frac{T}{2N}$ for all $\tau \geq \tau_0$ such that $\tau$ is not the last stage.
We now prove that $v_i \geq \frac{1}{2 }\sqrt{\frac{1}{NT}}$. This is because, suppose the contrary, for $T$ such that $T \geq N^4$ and greater than a sufficiently large universal constant, conditioned on $\event^{(1)}$, we have that
\begin{align*}
&\hv_{i, \tau_0} \leq v_i + \sqrt{\frac{144 \ln(\sqrt{N}T^2+1)}{T_i^{(\tau_0)} / v_i}}
+ \frac{144\ln(\sqrt{N}T^2+1)}{T_i^{(\tau_0)}} \\
&\leq  {\frac{1}{2\sqrt{NT}}}
+ O\Big(\sqrt{\frac{ \ln(\sqrt{N}T^2+1)}{\sqrt{T^3/N}}}
+ \frac{\ln(\sqrt{N}T^2+1)}{T}\Big),
\end{align*}
which is at most $1/ \sqrt{NT}$,
contradicting to the condition $\mathcal{P}(i, \tau_0)$ and that $\tau_0$ is not the last stage.

Moreover, for $T$ such that $T \geq N^4$ and greater than a sufficiently large universal constant, when $\tau > \tau_0$, using $T^{(\tau)}_i \geq |\mathcal{T}(i, \tau_0)|  \geq \frac{T}{2N}$, we have that 
\begin{align*}
\hv_{i, \tau}  \leq v_i + \sqrt{\frac{144 v_i\ln(\sqrt{N}T^2+1)}{T_i^{(\tau)}}} + \frac{144\ln(\sqrt{N}T^2+1)}{T_i^{(\tau)}}  \leq 2v_i.
\end{align*}
 By the condition $\mathcal{P}(i, \tau)$, when $\tau > \tau_0$ and $\tau$ is not the last stage, we have that
\[
|\cT(i, \tau_i)| \geq 1 + \sqrt{\frac{T\cdot T_i^{(\tau_i)}}{N\cdot \hv_{i, \tau_i}}} \geq 1 + \sqrt{\frac{T\cdot |\cT(i, \tau_i-1)|}{2N\cdot v_i}}.
\]
Applying Fact~\ref{fact:seq}, we prove the desired inequality of this lemma.
\end{proof}

\subsection{Proof of Lemma~\ref{lem:regret in stage}} \label{app:proof-lem-regret-in-stage}
\begin{proof}[Proof of Lemma~\ref{lem:regret in stage}]
For the first stage, i.e., $\tau = 1$, since the number of epochs in this stage is at most $\sqrt{T/N}$, we have that
$\sum_{\ell \in \cT(i, 1)} (\hv_{i, 1}-v_i) \leq \sqrt{T/N}$
for any item $i$. From now on, we only prove the lemma for $\tau \in [2, \tau_i(L)]$. 

If $\tau \in [2, \tau_0]$, we have that $|\cT(i, \tau)| \leq \sqrt{\frac{T \cdot T_i^{(\tau)}}{N}} + 1$. By $\event^{(1)}$, we upper bound  $\sum_{\ell \in \cT(i, \tau)} (\hv_{i, \tau}-v_i) $ by the order of
\begin{align*}
 \sqrt{\frac{T \cdot T_i^{(\tau)}}{N}} 
\cdot \Bigg(\sqrt{\frac{v_i\ln(\sqrt{N}T^2+1)}{T_i^{(\tau)}}}
+\frac{\ln(\sqrt{N}T^2+1)}{T_i^{(\tau)}}\Bigg)
\lesssim \sqrt{T\ln(\sqrt{N}T^2+1)/N}, 
\end{align*}
where the inequality holds due to that $v_i \leq 1$ and $T_i^{(\tau)} \geq \sqrt{T/N}$ for any $\tau \in [2, \tau_0]$ (by Lemma~\ref{lem:stage-bound-1}).

When $\tau > \tau_0$, we prove the lemma by considering the following two cases. The first case is that $\hv_{i, \tau_0} \leq 1/\sqrt{NT}$. In this case, we have that
\[\textstyle
\sum_{\ell \in \cT(i, \tau)} (\hv_{i, \tau}-v_i) 
\leq T \cdot \hv_{i, \tau}
\leq \sqrt{T/N}.
\]
In the second case where $\hv_{i, \tau_0} > 1/\sqrt{NT}$, by Lemma~\ref{lem:stage-bound-2} it holds that $v_i \geq 1/(2\sqrt{NT})$. By $\event^{(1)}$, we have $\hv_{i, \tau} \geq v_i$. Therefore, $\hv_{i, \tau} \geq  1/ (2\sqrt{NT})$. Also note that $T^{(\tau)}_i \geq |\mathcal{T}(i, \tau_0)|  \geq \frac{T}{2N}$ by Lemma~\ref{lem:stage-bound-1}, and $|\cT(i, \tau)| \leq 1+\sqrt{\frac{T \cdot T_i^{(\tau)}}{N \cdot \hv_{i, \tau}}}$. Altogether, we have that $\sum_{\ell \in \cT(i, \tau)} (\hv_{i, \tau}-v_i)$  is upper bounded by a universal constant times
\begin{multline*}
 \sqrt{\frac{T \cdot T_i^{(\tau)}}{N \cdot \hv_{i, \tau}}} 
\cdot \left(\sqrt{\frac{v_i \ln(\sqrt{N}T^2+1)}{T_i^{(\tau)}}} 
+ \frac{\ln(\sqrt{N}T^2+1)}{T_i^{(\tau)}}\right)
\lesssim \sqrt{\frac{T  \ln(\sqrt{N}T^2+1)}{N}}  
+ \frac{\sqrt{T}\ln(\sqrt{N}T^2+1)}{\sqrt{N T_i^{(\tau)} \hv_{i, \tau}}},
\end{multline*}
which is $O(\sqrt{T\ln(\sqrt{N}T^2+1)/N})$ for $T \geq N^4$.
\end{proof}

\section{Bounding the number of item switches for Algorithm~\ref{alg:doubling-ucb}} \label{app:proof-of-theorem-FH-DUCB}

Since an assortment switch may incur at most $2K$ item switches, Theorem~\ref{thm:ucb-doubling-regret} trivially implies that Algorithm~\ref{alg:doubling-ucb} (AT-DUCB) incurs at most $O(KN \log T)$ item switches, which is upper bounded by $O(N^2\log T)$ since $K = O(N)$. In the following theorem, we prove an improved upper bound on item switches for Algorithm~\ref{alg:doubling-ucb}. 

\begin{theorem}\label{thm:ucb-doubling-item-switch}
For any input instance with $N$ items, before any time $T$, the number of item switches of Algorithm~\ref{alg:doubling-ucb} (AT-DUCB) satisfies that $\IS_T \lesssim N^{1.5} \log T$.
\end{theorem}

The proof of Theorem~\ref{thm:ucb-doubling-item-switch} includes a novel analysis with the careful application of the Cauchy-Schwartz inequality, which will be presented immediately after this paragraph. However, we would like to first add a few remarks on the optimality of the presented analysis. Indeed, we do not know whether the upper bound proved in Theorem~\ref{thm:ucb-doubling-item-switch} can be improved, and leave the possibility of further improvement as an open question. Our preliminary research suggests that the number of the item switches of Algorithm~\ref{alg:doubling-ucb} is closely related to the maximal number of planar $K$-sets (i.e., the number of subsets $P' \subseteq P$ where $P$ is a given set of $N$ points in a 2-dimensional plane, $P' = P \cap H$ for a half-space $H$). Very roughly, this relation is suggested by Lemma~\ref{lem:optimal-assortment}, where the optimal assortment $\arg\max_{S\subseteq[N], |S| \leq K}R(S,\bm{v})$ can be viewed as a planar $K$-set whether each item correspond to a 2-dimensional point $(-v_i, v_i r_i)$ and the half plane $H = \{(x, y) : y \ge r^{\star}  \cdot x + b\}$ for some parameter $b$. The continuous change of the the estimated optimal revenue $r^{\star}$ during the UCB algorithm may produce many half planes, and lead to the item change in the $K$-sets (assortments). Upper bounding the number of the $K$-sets would result in an upper bound for the number of the item switches. To our best knowledge, the best known upper bound for the number of planar $K$-sets is $O(NK^{1/3})$ \cite{dey1998improved}, and the best known lower bound is $N e^{\Omega(\sqrt{\log K})}$ \cite{toth2001point}. For future work, it is very interesting to study whether these upper and lower bounds imply the bounds on the number of item switches of our Algorithm~\ref{alg:doubling-ucb}.

Now we dive into the proof of Theorem~\ref{thm:ucb-doubling-item-switch}.

We first analyze the optimization process of $\arg\max_{S\subseteq[N], |S| \leq K}R(S,\bm{v})$ for any preference vector ${\bm{v}}$. Define $F(\bm{v}) \defeq\max_{S\subseteq[N], |S| \leq K}R(S,\bm{v})$. The following lemma characterizes the optimal assortment $S$ given the preference vector $\bm{v}$. Similar statements can also be found in, e.g., Section 2.1 of \cite{RSS10}.

\begin{lemma}\label{lem:optimal-assortment}
For any preference value vector $\bm{v}\ge 0$, let $r^\star=F(\bm{v}).$ Define $g_i=v_{i}(r_{i}-r^\star).$ Let $\sigma$ be the minimal permutation of $[N]$ such that $g_{\sigma_i}\ge g_{\sigma_j}$ for all $1\le i<j\le N.$ (In other words, $\sigma$ is the sorted index according to value $g$, with a deterministic tie-breaking rule). Then the optimal assortment $S$ is given by $S=\{\sigma_i:1\le i\le K, g_{\sigma_i}>0\}.$
\end{lemma}
\begin{proof}
Let $\optS=\arg\max_{S\subseteq[N], |S| \leq K}R(S,\bm{v})$. Then we have
$$\frac{\sum_{i\in\optS}r_iv_i}{1+\sum_{i\in\optS} v_i}=r^\star,$$
which implies that 
\begin{equation}\sum_{i\in\optS}v_i(r_i-r^\star)=\sum_{i\in\optS} g_i=r^\star.\end{equation}
Now we prove that $\optS=\arg\max_{S\subseteq[N], |S| \leq K}\left(\sum_{i\in S} g_i\right)$. Suppose otherwise that there exists $S'\subseteq[N]$ with $|S'| \leq K$ such that $\sum_{i\in S'}g_i>\sum_{i\in \optS}g_i=r^\star.$ It follows that $\sum_{i\in S'}v_i(r_i-r^\star)>r^{\star}.$ Therefore,
$$R(S',\bm{v})=\frac{\sum_{i\in S'}v_ir_i}{1+\sum_{i\in S'}v_i}>r^{\star},$$ which contradicts to the definition of $S^\star.$

Now, note that $\sigma$ is a permutation of $[N]$ such that $g_{\sigma_i}$ is non-increasing according to $i$. We have that $\arg\max_{S\subseteq[N], |S| \leq K}\left(\sum_{i\in S} g_i\right)=\{\sigma_i:1\le i\le K, g_{\sigma_i}>0\}$, which finishes the proof.
\end{proof}

The next lemma shows that $F(\bm{v})$ is monotonically decreasing in $\bm{v}$.
\begin{lemma}\label{lem:monotonicity}
Consider two vectors $\bm{v}$ and $\hat{\bm{v}}$. If $\hat{v}_i\ge v_i\ge 0$ for all $i\in [N]$, we have $F(\hat{\bm{v}})\ge F(\bm{v}).$
\end{lemma}
\begin{proof}
Let $S^\star=\arg\max_{S\subseteq[N], |S| \leq K}R(S,\bm{v})$ and $r^\star=R(S^\star,\bm{v}).$ Then we have $\sum_{i\in S^\star}v_i(r_i-r^\star)=r^\star.$ According to Lemma~\ref{lem:optimal-assortment}, $r_i-r^{\star}>0$ for all $i\in S^\star$. Combining with the assumption that $\hat{v}_i\ge v_i,\forall i\in[N]$, we get
$\sum_{i\in S^\star}\hat{v}_i(r_i-r^{\star})\ge \sum_{i\in S^\star}v_i(r_i-r^\star)=r^\star.$ As a result, $$R(S^\star,\hat{v})=\frac{\sum_{i\in S^\star}r_iv_i}{1+\sum_{i\in S^\star}v_i}\ge r^\star.$$ Therefore, $F(\hat{\bm{v}})=\max_{S\subseteq[N], |S| \leq K}R(S,\hat{\bm{v}})\ge R(S^\star,\hat{\bm{v}})\ge r^\star=F(\bm{v})$.
\end{proof}

Let $m$ be the total number of times that Line~\ref{line:doubling} of Algorithm~\ref{alg:doubling-ucb} is executed, and let $\tau^{(1)} < \tau^{(2)} < \tau^{(3)} < \dots < \tau^{(m)}$ be the time steps that Line~\ref{line:doubling} of Algorithm~\ref{alg:doubling-ucb} is executed. In other words, only in the time steps in $\{\tau^{(p)}\}_{p=0}^{m}$, the UCB value vector $\hat{\bm{v}}$ is updated (where for convenience, we set $\tau^{(0)} = 0$). Let $\htv^{(p)}$ be the UCB value after the update at time $\tau^{(p)},$ and for convenience we let $\htv^{(0)}=(1,1,\cdots,1)$. Define $r^{(p)}=F(\htv^{(p)})$. Let $\rho^{(p)}_i$ be the rank of item $i$ according to value $g_i^{(p)}\defeq \hv^{(p)}_i(r_i-r^{(p)})$ with the tie-breaking rule defined in Lemma~\ref{lem:optimal-assortment}. We then have the following lemma.
\begin{lemma}\label{lem:inversion}
Let $\delta_{i,j}^{(p)}\defeq \ind[\rho^{(p)}_i>\rho^{(p)}_j]$. For any two items $i,j\in[N]$, the number of times that the relative order of $i,j$ changes is bounded by $c\log T$ for some universal constant $c$. That is,
$$\sum_{p=0}^{m-1}\ind\left[\delta_{i,j}^{(p)}\neq \delta_{i,j}^{(p+1)}\right]\lesssim \log T.$$
As a corollary, we have that $$\sum_{i,j\in[N]}\sum_{p=0}^{m-1}\ind\left[\delta_{i,j}^{(p)}\neq \delta_{i,j}^{(p+1)}\right]\lesssim N^2\log T.$$
\end{lemma}
\begin{proof}
Let $\mathcal{D}_{i}^{(p)}$ be the event that Line~\ref{line:doubling} is executed in Algorithm~\ref{alg:doubling-ucb} for item $i$ at time $\tau^{(p)}.$ In the following we prove that $$\sum_{p=0}^{m-1}\ind\left[\delta_{i,j}^{(p)}\neq \delta_{i,j}^{(p+1)}\right]\le 2\sum_{p=0}^{m-1}\mathcal{D}_{i}^{(p)}+2\sum_{p=0}^{m-1}\mathcal{D}_{j}^{(p)}.$$

For a fixed pair of items $i,j$, let $\{\bar{p}_q\}_{q=1}^{Q}$ be the time steps that $\mathcal{D}_{i}^{(\bar{p}_q)}$ or $\mathcal{D}_{j}^{(\bar{p}_q)}$ occur. We only need to prove that $$\sum_{p=\bar{p}_q}^{\bar{p}_{q+1}-1}\ind\left[\delta_{i,j}^{(p)}\neq \delta_{i,j}^{(p+1)}\right]\le 1$$ for all $q\in[Q]$.

Note that at time interval $[\bar{p}_q, \bar{p}_{q+1}-1]$, $\bv_i$ and $\bv_j$ does not change. Therefore, $\delta_{i,j}^{(p)}=\ind[\bv_i(r_i-r^{(p)})< \bv_j(r_j-r^{(p)})].$ It is implied by Lemma~\ref{lem:monotonicity} that $r^{(p)}$ is monotonically decreasing. As a result, $\sum_{p=\bar{p}_q}^{\bar{p}_{q+1}-1}\ind\left[\delta_{i,j}^{(p)}\neq \delta_{i,j}^{(p+1)}\right]\le 1$.
\end{proof}

Now we are ready to prove Theorem~\ref{thm:ucb-doubling-item-switch}.
\begin{proof}[Proof of Theorem~\ref{thm:ucb-doubling-item-switch}]
Let $K^{(p)}=\min\left\{K,\left|\{i:g_i^{(p)}>0\}\right|\right\}.$ Note that since $r^{(p)}$ is non-increasing, $K^{(p)}$ is non-decreasing. Then we have, $S^{(\tau_p)}=\{i:\rho_i^{(p)}\le K^{(p)}\}.$ Let $\bar{S}^{(\tau_{p+1})}=\{i:\rho_i^{(p+1)}\le K^{(p)}\}.$ Then we have, $\bar{S}^{(\tau_{p+1})}\subseteq S^{(\tau_{p+1})}$ and $\left|S^{(\tau_{p+1})}\setminus \bar{S}^{(\tau_{p+1})}\right|=K^{(p+1)}-K^{(p)}.$ It follows that
\begin{equation}
\left|S^{\tau_{p}}\oplus S^{\tau_{p+1}}\right|\le \left|S^{\tau_{p}}\oplus \bar{S}^{\tau_{p+1}}\right|+K^{(p+1)}-K^{(p)}.
\label{equ:fix-size}
\end{equation}
%\blue{might be a separate lemma.}
Let $x^{(p)}=\left|S^{\tau_{p}}\oplus \bar{S}^{\tau_{p+1}}\right|.$ In the following we prove that 
\begin{equation}(x^{(p)}/2)^2\le \sum_{i,j\in[N]}\ind[\delta_{i,j}^{(p)}\neq \delta_{i,j}^{(p+1)}].\label{equ:switch-to-inversion}
\end{equation}
Note that $|S^{(\tau_{p})}|=|\bar{S}^{(\tau_{p+1})}|=K^{(p)}.$ Define $Z=S^{(\tau_{p})}\setminus \bar{S}^{(\tau_{p+1})}$ and $Z'=\bar{S}^{(\tau_{p+1})}\setminus S^{(\tau_{p})}$. Then we have that $x^{(p)}=2|Z|=2|Z'|.$ Note that for all $i\in Z$, we have that $\rho_i^{(p)}\le K^{(p)}$ and $\rho_i^{(p+1)}>K^{(p)}$. And for all $j\in Z'$, we have that $\rho_i^{(p)}> K^{(p)}$ and $\rho_i^{(p+1)}\le K^{(p)}.$ It follows that $\delta_{i,j}^{(p)}=0,\delta_{i,j}^{(p+1)}=1$ for all $i\in Z,j\in Z'.$ Hence, we have that 
\[
\sum_{i,j\in [N]}\ind[\delta_{i,j}^{(p)}\neq \delta_{i,j}^{(p+1)}]\ge |Z| \times |Z'|=(x^{(p)}/2)^2,
\]
which establishes \eqref{equ:switch-to-inversion}.

Combining \eqref{equ:switch-to-inversion} and Lemma~\ref{lem:inversion}, we have that
$\sum_{p=1}^{m-1}(x^{(p)}/2)^2\le N^2\log T.$ By the deferred update rule in Algorithm~\ref{alg:doubling-ucb}, we have that $m\le N(1+\log T).$ Applying Cauchy-Schwarz inequality, we get that $$\sum_{p=1}^{m-1}x^{(p)}\lesssim N^{1.5}\log T.$$

Therefore, by \eqref{equ:fix-size} we have that \begin{equation}\sum_{p=1}^{m-1}|S^{(\tau_p)}\oplus S^{(\tau_{p+1})}|\le \sum_{p=1}^{m-1}(x^{(p)}+K^{(p+1)}-K^{(p)})\lesssim N^{1.5}\log T.\label{equ:thm3-final}
\end{equation} Note that there is no assortment switch at time steps where $\htv$ is not updated. Therefore \eqref{equ:thm3-final} directly leads to Theorem~\ref{thm:ucb-doubling-item-switch}.
\end{proof}

\section{Omitted proofs for the ESUCB algorithm in Section~\ref{sec:IS}} \label{app:ESUCB}

\subsection{Proof of Lemma~\ref{lem:characterization-of-G}} \label{app:proof-lem-char-G}

\begin{proof}[Proof of Lemma~\ref{lem:characterization-of-G}]
We first prove the existence of $\optt$. Note that the uniqueness follows directly from statements 1) and 2) in the lemma statement.
\paragraph{Proof of the existence of $\optt$.} Let $S^\star=\arg\max_{S\subseteq[N] : |S| \leq K}R(S,\bm{v})$ and $\optt=R(S^\star,\bm{v})$. We only need to prove that $G(\optt)=\optt.$

On the one hand, since $G(\theta)=R(S_\theta,\bm{v})$, we have $G(\optt)\le \optt$ be the optimality of $S^\star$. On the other hand, we will prove that $G(\optt)\ge \optt$. For the sake of contradiction, suppose $G(\optt)<\optt.$ Then we have,
$$\frac{\sum_{i\in S_{\optt}}v_ir_i}{1+\sum_{i\in S_{\optt}}v_i}=G(\optt)<\optt.$$ By algebraic manipulation we get $\sum_{i\in S_{\optt}}v_i(r_i-\optt)<\optt.$ By the optimality of $S_{\optt}$ we have
$$\sum_{i\in \optS}v_i(r_i-\optt)\le \sum_{i\in S_{\optt}}v_i(r_i-\optt)<\optt.$$ As a result, we have $R(\optS,\bm{v})=\frac{\sum_{i\in \optS}v_ir_i}{1+\sum_{i\in \optS}v_i}<\optt,$ which leads to contradiction.

\paragraph{Proof of statement 1).} For the sake of contradiction, suppose $G(\theta)\le \theta.$ Then we have $$\frac{\sum_{i\in S_\theta}r_iv_i}{1+\sum_{i\in S_\theta}v_i}\le \theta,$$ which means that $\sum_{i\in S_\theta} v_i(r_i-\theta)\le \theta.$ Note that $v_i\ge 0$ for all $i\in [N]$. By the optimality of $S_\theta$, we get
$$\sum_{i\in S_{\optt}}v_i(r_i-\optt)\le \sum_{i\in S_{\optt}}v_i(r_i-\theta)\le \sum_{i\in S_{\theta}}v_i(r_i-\theta)\le \theta<\optt.$$ By algebraic manipulation, we get $R(S_{\optt},\bm{v})<\optt,$ which leads to contradiction.

\paragraph{Proof of statement 2).} By the optimality of $\optS$, we have $G(\theta)\le G(\optt)=\optt<\theta$.
\end{proof}

\subsection{Proof of Lemma~\ref{lem:item-switch-check}}\label{app:proof-lem-item-switch-check}

\begin{proof}[Proof of Lemma~\ref{lem:item-switch-check}]
Observe that in the \chk procedure, when $b$ equals $\false$, $S_{\ell}$ is  evaluated by Line~\ref{line:erm-with-fixed-theta} and with respect to $\theta_r$. When $b$ is set to $\true$, $S_{\ell}$ will always be evaluated by Line~\ref{line:erm-with-fixed-theta-0} with respect to $\theta_l$. This switch happens for at most once. Therefore, we only need to show that for fixed any $\theta \in \{\theta_l, \theta_r\}$, and $S_{\ell}' = \mathop{\arg\max}_{S \subseteq [N], |S| \leq K}\left(\sum_{i\in S}\hv_i(r_i-\theta)\right)$, it holds that (assuming that there are $L$ epochs)
\begin{align}\label{eq:lem-item-switch-check-goal}
%\textstyle
\sum_{\ell = 1}^{L - 1} |S_{\ell}' \oplus S_{\ell + 1}'| \lesssim N \log T .
\end{align}
Suppose that there are $n_{\ell}$ items whose UCB values are updated after the $\ell$-th epoch. We claim that $|S_{\ell} \oplus S_{\ell + 1}| \leq n_{\ell}$. This is simply because $S_{\ell}$ corresponds to the items $i \in [N]$ such that $\htv_i (r_i - \theta)$ is positive and among the $K$ largest ones (and thanks to the tie breaking rule). Therefore, any update to a single $\htv_i$ will incur at most one item switch to $S_{\ell}$, and $n_{\ell}$ updates will incur at most $n_{\ell}$ item switches. Now, \eqref{eq:lem-item-switch-check-goal} is established because
$\sum_{\ell = 1}^{L - 1} |S_{\ell}' \oplus S_{\ell + 1}'|  \leq \sum_{\ell = 1}^{L-1} n_\ell \lesssim N \log T$,
where the second inequality is due to the deferred update rule for the UCB values.
\end{proof}

\subsection{Proof of Lemma~\ref{lem:check-main}}\label{app:proof-lem-check-main}

We now prove Lemma~\ref{lem:check-main}. For preparation, we first show that the UCB value $\hat{v}_i$ is valid throughout the execution of Algorithm~\ref{alg:check}.
\begin{lemma}\label{lem:ucb-value-check}
For any invocation of $\chk(\theta_l,\theta_r,\tmax)$, and for any epoch $\ell=1,2,3,\dots,$ during the algorithm, the following two statements hold throughout the execution,
\begin{itemize}
	\item[1.] With probability at least $1-\frac{\delta}{4NT^2}$, $\hat{v}_{i}^{(\ell)}\ge v_i$ for any $i\in [N],$
	\item[2.] With probability at least $1-\frac{\delta}{4NT^2}$, for any $i\in [N],$
	$$\hat{v}_{i}^{(\ell)}-v_i\le \sqrt{\frac{196v_i\log(NT/\delta)}{T_i^{(\ell)}}}+\frac{292\log(NT/\delta)}{T_i^{(\ell)}}.$$
\end{itemize}
\end{lemma} 
\begin{proof}
The proof is essentially the same as Lemma~\ref{lem:ucb-value}. 
\end{proof}
Let $\mathcal{H}$ be the event that the events described by Lemma~\ref{lem:ucb-value-check} holds throughout the execution of Algorithm~\ref{alg:check} for any $\ell$ and $i\in [N]$. We have that $\Pr[\mathcal{H}]\ge 1-\frac{\delta}{4T}.$

Now we prove the following lemma.
\begin{lemma}\label{lem:regret-by-ucb}
For any fixed $\theta$ where $G(\theta)\ge \theta,$ define $\hat{S}_\theta=\arg\max_{S:S\subseteq [N], |S| \leq K}\left(\sum_{i\in S}\hat{v}_i(r_i-\theta)\right).$ Suppose $\hat{v}_i\ge v_i$ for all $i \in [N]$. We have that
$$\left(1+\sum_{i\in \hat{S}_\theta}v_i\right)\left(\theta-R(\hat{S}_\theta,\bm{v})\right)\le \sum_{i\in \hat{S}_\theta}(\hat{v}_i-v_i).$$
\end{lemma}
\begin{proof}
Recall that $S_\theta=\arg\max_{S:S\subseteq [N], |S| \leq K}\left(\sum_{i\in S}v_i(r_i-\theta)\right)$. We then have that
\begin{align}
&\left(1+\sum_{i\in \hat{S}_\theta}v_i\right)\left(\theta-R(\hat{S}_\theta,\bm{v})\right)\nonumber\\
=\;&\left(1+\sum_{i\in \hat{S}_\theta}v_i\right)\left(\theta-\frac{\sum_{i\in \hat{S}_\theta}r_i\hat{v}_i}{1+\sum_{i\in \hat{S}_\theta}\hat{v}_i}+\frac{\sum_{i\in \hat{S}_\theta}r_i\hat{v}_i}{1+\sum_{i\in \hat{S}_\theta}\hat{v}_i}-R(\hat{S}_\theta,\bm{v})\right)\nonumber \\
=\;&\left(1+\sum_{i\in \hat{S}_\theta}v_i\right)\left(\theta-\frac{\sum_{i\in \hat{S}_\theta}r_i\hat{v}_i}{1+\sum_{i\in \hat{S}_\theta}\hat{v}_i}\right)+\sum_{i\in \hat{S}_\theta}r_i\left(\left(1+\sum_{i\in \hat{S}_\theta}v_i\right)\frac{\hat{v}_i}{1+\sum_{i\in \hat{S}_\theta}\hat{v}_i}-v_i\right) \label{equ:regret-by-ucb-key}.
\end{align}
Note that by assumption we have $\hat{v}_i\ge v_i$ for all $i \in [N]$. Therefore it holds that $1+\sum_{i\in \hat{S}_\theta}\hat{v}_i\ge 1+\sum_{i\in \hat{S}_\theta}v_i.$ As a result, \begin{equation}\label{equ:regret-by-ucb-key-1}
\sum_{i\in \hat{S}_\theta}r_i\left(\left(1+\sum_{i\in \hat{S}_\theta}v_i\right)\frac{\hat{v}_i}{1+\sum_{i\in \hat{S}_\theta}\hat{v}_i}-v_i\right)\le \sum_{i\in \hat{S}_\theta}r_i\left(\hat{v}_i-v_i\right)\le \sum_{i\in \hat{S}_\theta}\left(\hat{v}_i-v_i\right).
\end{equation}
On the other hand, 
\begin{equation}\label{equ:regret-by-ucb-key-2.1}
\left(1+\sum_{i\in \hat{S}_\theta}v_i\right)\left(\theta-\frac{\sum_{i\in \hat{S}_\theta}r_i\hat{v}_i}{1+\sum_{i\in \hat{S}_\theta}\hat{v}_i}\right)=\frac{1+\sum_{i\in \hat{S}_\theta}v_i}{1+\sum_{i\in \hat{S}_\theta}\hat{v}_i}\left(\theta-\sum_{i\in \hat{S}_\theta}\hat{v}_i(r_i-\theta)\right).
\end{equation}
Note that by monotonicity (see Lemma~\ref{lem:monotonicity}) and our assumption (namely, $G(\theta)>\theta)$, $$\frac{\sum_{i\in \hat{S}_\theta}r_i\hat{v}_i}{1+\sum_{i\in \hat{S}_\theta}\hat{v}_i}=R(\hat{S}_\theta,\htv)\ge R(S_\theta,\bm{v})= G(\theta)\ge \theta.$$ By algebraic manipulation, we get that
\begin{equation}\label{equ:regret-by-ucb-key-2.2}
\sum_{i\in \hat{S}_\theta}\hat{v}_i(r_i-\theta) \ge \theta.
\end{equation}
Combining \eqref{equ:regret-by-ucb-key-2.1} and \eqref{equ:regret-by-ucb-key-2.2}, we get that
\begin{equation}\label{equ:regret-by-ucb-key-2}
\left(1+\sum_{i\in \hat{S}_\theta}v_i\right)\left(\theta-\frac{\sum_{i\in \hat{S}_\theta}r_i\hat{v}_i}{1+\sum_{i\in \hat{S}_\theta}\hat{v}_i}\right)\le 0.
\end{equation}
Plug in \eqref{equ:regret-by-ucb-key-1} and \eqref{equ:regret-by-ucb-key-2} into \eqref{equ:regret-by-ucb-key}, we have that
$$\left(1+\sum_{i\in \hat{S}_\theta}v_i\right)\left(\theta-R(\hat{S}_\theta,\bm{v})\right)\le \sum_{i\in \hat{S}_\theta}(\hat{v}_i-v_i).$$
\end{proof}

We will also need the following Azuma-Hoeffding inequality for martingales.
\begin{theorem}\label{thm:azuma}
Suppose $\{X_k : k = 0, 1, 2, 3, \dots, \}$ is a martingale and $|X_k - X_{k-1}| \leq M$ almost surely for all $k$. Then for all positive integers $n$ and all positive reals $\epsilon$, it holds that
\[
\Pr[X_n - X_0 \geq \epsilon] \leq \exp\left(-\frac{\epsilon^2}{2nM^2}\right).
\]
\end{theorem}

Now we are ready to prove Lemma~\ref{lem:check-main}.
\begin{proof}[Proof of Lemma~\ref{lem:check-main}] We prove that each of the statements (a)--(c) holds with probability at least $1-\delta/(4T)$, given that the UCB estimation of value $\bm{v}$ is valid (i.e., event $\mathcal{H}$). Then Lemma~\ref{lem:check-main} holds by a union bound.
\paragraph{Proof of statement (a).} Note that we only need to prove that if $G(\theta_r)\ge \theta_r$, then with probability at least $1-\delta/(4T),$ $\chk(\theta_l,\theta_r,\tmax)$ returns $\false$.

For simplicity, we use the superscript ${(\ell)}$ to denote the value of a variable in Algorithm~\ref{alg:check} at the beginning of epoch $\ell$. For example, $t^{(\ell)}$ denotes the time steps taken at the beginning of epoch $\ell.$ Now we prove that for large enough constants $c_2$ and $c_3$, and any fixed $L$ it holds that
\newcommand{\tnow}{t^{(L)}}
\begin{align}
\Pr\Big[\sum_{\tau=1}^{t^{(L)}}\left(R(S_{\theta_r}^{(\tau)}, \bm{v})-\theta_r\right)+&(c_2-8)\sqrt{Nt^{(L)} \log^3(NT/\delta)}\nonumber \\
&+c_3N\log^3(NT/\delta)\ge 0\wedge t^{(L)}\le \tmax\Big]\le 1-\delta/(8T).\label{equ:lemma8-a-key}
\end{align}
Let $\mathcal{J}_\ell$ be the filtration of random variables upto epoch $\ell$. Let $S_\theta^{(\ell)}=\arg\max_{S:S\subseteq[N], |S| \leq K}\left(\sum_{i\in S}r_i(\hat{v}_i^{(\ell)}-\theta)\right).$ Then $S_{\theta_r}^{(\ell)}$ is $\mathcal{J}_{\ell-1}$ measurable. %Therefore we have, the sequence $r^{(\tau)}_{\chk}-R(S_{\theta_r^{(\tau)}},\bm{v})$ is a martingale sequence. By Theorem~\ref{thm:azuma} and union bound (over all possible value of $L$), with probability at least $1-\delta/(8T)$, $$\sum_{\tau=1}^{t^{(L)}}r_{\chk}^{(\tau)}\le \sum_{\tau=1}^{t^{(L)}}R(S_{\theta_r^{(\tau)}},\bm{v})+8\sqrt{t^{(L)}\log(T/\delta)}.$$
%Recall that $S_{\theta_r}^{(\tau)}$ is the same for $\tau\in [t^{(\ell)}, t^{(\ell+1)})$. 
For simplicity we define $S_\ell=S_{\theta_r}^{(\ell)}$. % for $\tau\in [t^{(\ell)}, t^{(\ell+1)})$.
As a result, $$\sum_{\tau=1}^{\tnow}\left(\theta_r^{(\ell)}-R(S_{\ell},\bm{v})\right)=\sum_{\ell=1}^{L}\left(t^{(\ell+1)}-t^{(\ell)}\right)\left(\theta_r^{(\ell)}-R(S_{\ell},\bm{v})\right).$$
Note that $\left(t^{(\ell+1)}-t^{(\ell)}\right)$ follows geometric distribution given $\mathcal{J}_{\ell-1}$ with mean $\left(1+\sum_{i\in S_\ell} v_i\right)$. Therefore with probability at least $1-\delta/(16T^3)$ we have $t^{(\ell+1)}-t^{(\ell)}\le 24\log(T/\delta)\left(1+\sum_{i\in S_\ell} v_i\right).$ Consequently, with probability at least $1-\delta/(16T^2),$
$$\sum_{\ell=1}^{L}\left(t^{(\ell+1)}-t^{(\ell)}\right)\left(\theta_r^{(\ell)}-R(S_{\ell},\bm{v})\right)\le \sum_{\ell=1}^{L}24\log(T/\delta)\left(1+\sum_{i\in S_\ell} v_i\right)\left(\theta_r^{(\ell)}-R(S_{\ell},\bm{v})\right)_+,$$
where the $(x)_+$ notation denotes $\max\left\{x,0\right\}.$
Under event $\mathcal{H}$, it follows from Lemma~\ref{lem:regret-by-ucb} that
\begin{align*}
&\sum_{\ell=1}^{L}24\log(T/\delta)\left(1+\sum_{i\in S_\ell} v_i\right)\left(\theta_r^{(\ell)}-R(S_{\ell},\bm{v})\right)_+\\
\le\; &24\log(T/\delta)\sum_{\ell=1}^{L}\sum_{i\in S_\ell} (\hat{v}_i^{(\ell)}-v_i)\\
\le\; &24\log(T/\delta)\sum_{\ell=1}^{L}\sum_{i\in S_\ell}\left(\sqrt{\frac{196v_i\log(NT/\delta)}{T_i^{(\ell)}}}+\frac{292\log(NT/\delta)}{T_i^{(\ell)}}\right)\\
\le\; &24\log(T/\delta)\left(\sum_{i\in [N]}\sqrt{392T_i^{(L)}v_i\log(NT/\delta)}+876N\log^2(NT/\delta)\right).
%\le\; &24\log(T/\delta)\left(\sum_{i\in [N]}\sqrt{392\E\left[T_i^{(L)}v_i\right]\log(NT/\delta)}+876N\log^2(NT/\delta)\right).
\end{align*}

Recall that in Algorithm~\ref{alg:check} we define $$\bar{v}_i^{(L)}=\sum_{\ell=1}^{L}\Delta_i^{(\ell)}/T_i^{(L)}.$$ Since $\Delta_i^{(\ell)}$ follows geometric distribution, by concentration inequality (namely, Theorem 5 of \cite{AAGZ17})
$$\Pr\left[\bar{v}_i^{(L)}<\frac{1}{2}v_i\right]\le \exp\left(-T_i^{(L)}v_i/48\right).$$
Therefore we get with probability at least $1-\delta/(16T^2)$, for any $i\in[N]$,
$$T_i^{(L)}v_i\le \max\left\{2\bar{n}_i^{(L)},144\log(NT/\delta)\right\}.$$
Since every time step at most one item can be chosen, we get $\sum_{i\in[N]}\bar{n}_i^{(L)}\le \tnow.$ Consequently, 
\begin{align*}
&\sum_{i\in [N]}\sqrt{T_i^{(L)}v_i\log (NT/\delta)}\\
\le\; &\sum_{i\in [N]}\sqrt{2\bar{n}_i^{(L)}\log (NT/\delta)}+\sqrt{144}N\log(NT/\delta)\\
\le\; &\sqrt{2Nt^{(L)}\log (NT/\delta)}+\sqrt{144}N\log(NT/\delta).
\end{align*}

Putting everything together, we prove Eq.~\eqref{equ:lemma8-a-key} with $c_2=688$ and $c_3=21036$.
Note that 
\[
r_{\chk}^{(\tau)} - R(S_{\theta_r}^{(\tau)},\bm{v})
%\E\left[r_{\chk}^{(\tau)}-\theta_r | r_{\chk}^{(1)}, r_{\chk}^{(2)}, \dots, r_{\chk}^{(\tau-1)}\right]
\]
is a martingale sequence for $\tau = 0, 1, 2, 3, \dots$. By Theorem~\ref{thm:azuma} (using $M = 2$), with probability $1-\delta/(8T^2)$, we have that 
\[
\sum_{\tau=1}^{\tnow}\left(r_{\chk}^{(\tau)}-\theta_r\right) \geq \sum_{\tau=1}^{\tnow}\left(R(S_{\theta_r}^{(\tau)})-\theta_r\right) - 8 \sqrt{\tmax \log(T/\delta)} .
\]
Combining with \eqref{equ:lemma8-a-key}, we get with probability at least $1 - \delta/(4T)$,   it holds that
$$\sum_{\tau=1}^{\tnow}\left(r_{\chk}^{(\tau)}-\theta_r\right)+c_2\sqrt{N\tmax \log^3(NT/\delta)}+c_3N\log^3(NT/\delta)\ge 0,$$ in any of the epoch $L$ such that $t^{(L)}\le \tmax$.
Consequently, with probability at most $1-\delta/(4T)$, the event that $\hat{\rho}^{(\ell)}<\theta$ never occur, which means that $\chk(\theta_l,\theta_r,\tmax)$ returns $\false$.

\paragraph{Proof of Statement (b).} Note that when the Algorithm returns $\false$, the \textbf{if}-condition in Line~\ref{line:if} is always $\false$. By the optimality, we have $\optt=G(\optt)\ge R(S_{\theta_r}^{(\tau)},\bm{v})$ for any $1\le \tau\le \tmax$. Note that $(r_{\chk}^{(\tau)}-R(S_{\theta_r}^{(\tau)},\bm{v}))$ is a martingale sequence. Again, invoking Theorem~\ref{thm:azuma}, we have that with probability at least $1-\delta/(8T)$, it holds that
\begin{align*}
\optt&\ge \frac{1}{t^{(L)}}\sum_{\tau=1}^{t^{(L)}}R(S_{\theta_r}^{(\tau)},\bm{v})\\
\ge\; &\frac{1}{t^{(L)}}\sum_{\tau=1}^{t^{(L)}}r_{\chk}^{(\tau)}-8 \sqrt{\log(T/\delta)/t^{(L)}}\tag{Martingale concentration}\\
\ge\; &\theta_r-\frac{1}{t^{(L)}}\left(c_2\sqrt{Nt^{(L)} \log^3(NT/\delta)}+c_3N\log^3(NT/\delta)+8\sqrt{t^{(L)}\log(T/\delta)}\right)\tag{By the \textbf{if} statement in Line~\ref{line:if}}.
\end{align*}
Note that the time steps taken by the last epoch is bounded by $24(N+1)\log(T/\delta)$ with probability $1-\delta/(8T)$. As a result, $(c_2+8)/t^{(L)}\le 2/\tmax$ and $c_3/t^{(L)}\le 2/\tmax.$ Consequently, 
\begin{align*}
&\theta_r-\frac{1}{t^{(L)}}\left(c_2\sqrt{Nt^{(L)} \log^3(NT/\delta)}+c_3N\log^3(NT/\delta)+8\sqrt{t^{(L)}\log(T/\delta)}\right)\\
\ge\; &\theta_r-\frac{2}{\tmax}\left(c_2\sqrt{N\tmax \log^3(NT/\delta)}+c_3N\log^3(NT/\delta)\right),
\end{align*}
which proves statement (b).
%where the last inequality is due to the fact that the time steps taken by the last epoch is bounded by $24\log(T/\delta)N$, which means that $t^{(L)}\ge 2\tmax/3.$ 

\paragraph{Proof of statement (c).} 
% $\tmax\theta_l-\E\left[\sum_{t=1}^{\tmax}r_{\chk}^{(t)}\right]\lesssim \sqrt{N\tmax\log (NT/\delta)}+N\log^2(NT/\delta).$
Let $\bar{t}$ be the time step when the \textbf{if} condition is first violated (and let $\bar{t}=\tmax$ if the condition holds throughout an execution). We first show that 
\begin{equation}\label{equ:lemma8-c-key1}
\E\left[\sum_{\tau=1}^{\bar{t}}\left(\theta_l-R(S_{\theta_r}^{(\tau)},\bm{v})\right)\right]\lesssim \sqrt{N\tmax\log^3 (NT/\delta)}+N\log^3(NT/\delta)
\end{equation}
holds with high probability. Note that the \textbf{if} condition is $\false$ for all $t\le \bar{t}$. Therefore, $ \bar{t}\theta_r \leq \sum_{\tau=1}^{\bar{t}}r_{\chk}^{(\tau)}+c_2\sqrt{N\tmax\log^3 (NT/\delta)}+c_3N\log^3(NT/\delta) $. Applying Theorem~\ref{thm:azuma}, we have that with probability at least $1 - \delta/(8T)$, it holds that $ \sum_{\tau=1}^{\bar{t}}r_{\chk}^{(\tau)}-\sum_{\tau=1}^{\bar{t}}R(S_{\theta_r}^{(\tau)},\bm{v})\lesssim \sqrt{\tmax\log (T/\delta)}$. Note that $\theta_l\le \theta_r$, we get \eqref{equ:lemma8-c-key1} with probability at least $1 - \delta/(8T)$.

Then we show that given $\bar{t},$
\begin{equation}\label{equ:lemma8-c-key2}
(\tmax-\bar{t})\theta_l-\E\left[\sum_{t=\bar{t}+1}^{\tmax}r_{\chk}^{(t)}\right]\lesssim \sqrt{N\tmax\log^3 (NT/\delta)}+N\log^3(NT/\delta),
\end{equation}
holds with high probability. Note that by assumption we have $\theta_l\le \theta^\star$. It follows from Lemma~\ref{lem:characterization-of-G} that $G(\theta_l)\ge \theta_l.$ By the same argument in the proof of statement (a), we have with probability $1-\delta/(8T)$, it holds that
$$\E\left[\sum_{\tau=\bar{t}+1}^{\tmax}\left(R(S_{\theta_l}^{(\tau)},\bm{v})-\theta_l\right)\right]+c_2 \sqrt{N\tmax \log^3(NT/\delta)}+c_3N\log^3(NT/\delta)\ge 0,$$ which implies \eqref{equ:lemma8-c-key2}.

Combining \eqref{equ:lemma8-c-key1} and \eqref{equ:lemma8-c-key2} with a union bound, we prove statement (c).
\end{proof}

\newcommand{\KL}{{\rm KL}}
\newcommand{\bT}{T_1}
\newcommand{\hT}{T_0^2}
\newcommand{\epsfunc}[1]{g(#1)}
\newcommand{\eventLB}{{\cal F}}
\newcommand{\hevent}{\hat{\eventLB}}
\newcommand{\instance}{{\cal I}}
\newcommand{\eventKnown}{{\cal G}}

\section{Lower bound proofs}

\subsection{Proof of \Cref{thm:lb-AS-logT}}
\label{app:lb-AS-logT}

%In this section, we show that without knowing the time horizon $T$ in advance, 
%no algorithm with regret $\tilde{O}(\sqrt{NT})$
%can achieve switch cost better than $O(\frac{N\log_2 T}{\log_2\log_2 NT})$.
To prove \Cref{thm:lb-AS-logT}, we first introduce the following more general theorem relating the expected regret with the number of assortment switches. 

\begin{theorem}\label{thm:lower}
For any $N\geq2$, $T_0 \geq 4$, fix a function $\epsfunc{T}$ such that $\epsfunc{T} \in \left[\frac{3}{\log_2 T}, \frac{1}{2}\right]$ and is non-increasing for $T \geq T_0$.
For any anytime algorithm, 
there exists an $N$-item assortment instance $\instance$ with time horizon $T\in [T_0, T^2_0]$ such that 
either the expected regret of the algorithm for instant $\instance$ is 
$$\expect{\regret_T} \geq  \frac{1}{7525} \cdot \sqrt{N}T^{\frac{1}{2} + \frac{\epsfunc{T}}{3}}$$
or the expected assortment switching cost before time $T$ is
$$\expect{\AS_T} = \expect{\sum_{t=1}^{T-1} \ind \left[S_t \neq S_{t+1} \right]} \geq \frac{N}{8\log_2 (1+\epsfunc{T})}.$$
\end{theorem}

Before proving \Cref{thm:lower}, we first prove \Cref{thm:lb-AS-logT} using \Cref{thm:lower}. 
\begin{proof}[Proof of \Cref{thm:lb-AS-logT}]
We set $\epsfunc{T} = \frac{3 C\ln\ln (NT)}{\ln T}$. 
It is easy to verify that the derivative of $\frac{\ln\ln (NT)}{\ln T}$
is 
\begin{align*}
\frac{\ln T - \ln (NT) \cdot \ln\ln (NT) }{T \ln^2 T \ln(NT)} < 0
\end{align*}
for all $N \geq 2$ and $T \geq 2$.  Therefore $\epsfunc{T}$ is non-increasing for all $N\geq 2$ and $T \geq 2$. Also note that for  $T \geq N$ and $T$ greater than a sufficiently large constant that only depends on $C$, we have that $\epsfunc{T} \in \left[\frac{3}{\log_2 T}, \frac{1}{2}\right]$.

Now invoke \Cref{thm:lower}, and we have that there exists an $N$-item assortment instance $\instance$ with time horizon $T\in [T_0, T^2_0]$ such that either $\expect{\regret_T} \geq\frac{1}{7525} \cdot \sqrt{NT} (\ln (NT))^C$ or 
\[
\expect{\AS_T} \geq \Omega\left(\frac{N}{\epsfunc{T}}\right) = \Omega\left(\frac{N \log T}{C \log \log (NT)} \right) ,
\]
proving  \Cref{thm:lb-AS-logT}.
\end{proof}

\begin{proof}[Proof of \Cref{thm:lower}]
Suppose that the expected number of assortment switches by the given policy for any input instance is at most 
$\frac{N}{8\log_2 (1+\epsfunc{T})}$ for any time horizon $T$, 
we will prove the theorem by showing that there exists an instance with time horizon $T\in [T_0, T^2_0]$ such that the expected regret is at least $\frac{1}{7525} \cdot T^{\frac{1}{2} + \frac{\epsfunc{T}}{3}}$. 

Consider the assortment instance $\instance = (\bm{v}, \bm{r})$, 
where $v_i = \frac{1}{2}$ and $r_i = 1$ for any $i\in [N]$.
We will let the capacity constraint  be $K=1$ for all assortment instances considered in this proof. 
By the assumption of the algorithm, the expected number of assortment switches given input instance $\instance$ is at most $\frac{N}{8\log_2 (1+\epsfunc{\hT})}$. 
Thus, there exists $\bT$ such that $\bT^{1+\epsfunc{\hT}} \in [T_0, \hT]$ and the expected number of assortment switches in time interval $[\bT, \bT^{1+\epsfunc{\hT}}]$ is at most $\frac{N}{8}$. 
Otherwise, there are $\frac{1}{\log_2 (1+\epsfunc{\hT})}$
such disjoint intervals in range $[T_0, \hT]$
and the expected number of assortment switches is at least $\frac{N}{8\log_2 (1+\epsfunc{\hT})}$, 
violating the assumption.  
Let 
\begin{align*}
\eventLB_1^{(i)} = \{ \text{item } i \text{ is not offered in time interval } [\bT, \bT^{1+\epsfunc{\hT}}] \text{ given instance } \instance\}.
\end{align*} 
Note that 
$\sum_i \Pr_{\instance}[\neg\eventLB_1^{(i)}] \leq \frac{N}{8}+1 \leq \frac{5N}{8}$ for any $N\geq 2$, 
because the expected number of items get offered in time interval $[\bT, \bT^{1+\epsfunc{\hT}}]$ is at most the expected number of assortment switches plus~1. 
Therefore, there must exist a set of items $I \subseteq [N]$
such that $|I| \geq \frac{N}{4}$ 
and for any item $i\in I$, $\Pr_{\instance}[\neg\eventLB_1^{(i)}] \leq \frac{5}{6}$.
% Hence for any item $i\in I$, 
% \begin{align}\label{eq:prob i}
% \Pr[\eventLB_1^{(i)}] \geq \frac{1}{6}.
% \end{align} 
Let 
\begin{align*}
\eventLB_2^{(i)} = \{\text{the number of times that item } i \text{ is offered in } [1, \bT] \text{ given instance } \instance \text{ is at most } \frac{48\bT}{N}\}.
\end{align*} 
Note that $T_1$ is at least the expected number of times an item $i\in I$ is chosen between $[1, T_1]$, 
which implies $\bT \geq \frac{48\bT}{N}\cdot \sum_{i\in I} \Pr_{\instance}[\neg\eventLB_2^{(i)}]$. 
Thus there exists $k\in I$ such that $\Pr_{\instance}[\neg\eventLB_2^{(k)}] \leq \frac{1}{12}$ since $|I| \geq \frac{N}{4}$. 
Let $\eventLB^{(k)} = \eventLB_1^{(k)} \cap \eventLB_2^{(k)}$, 
we have 
\begin{align}\label{eq:prob of k}
\Pr_{\instance}[\eventLB^{(k)}] \geq 1 - \Pr_{\instance}[\neg\eventLB_1^{(k)}] - \Pr_{\instance}[\neg\eventLB_2^{(k)}] \geq \frac{1}{12}.
\end{align}

Now we consider the assortment instance $\instance^{(k)} = (\bm{v}^{(k)}, \bm{r})$ 
where $v^{(k)}_{k} = \frac{1}{2} + \frac{1}{16}\sqrt{\frac{N}{24\bT}}$ and $v^{(k)}_j = \frac{1}{2}$ for $j\neq k$. 
We will be interested in the regret of the algorithm at time horizon $ \bT^{1+\epsfunc{\hT}}$. First, we show that with high probability, 
no algorithm can distinguish instance $\instance$ and $\instance^{(k)}$ at time $\bT$
with high probability. 
Formally, we have the following lemma, the proof of which is provided at the end of this section.
\begin{lemma}\label{lem:prob difference}
We have that 
\[
\left| \Pr_{\instance}[\eventLB^{(k)}]  - \Pr_{\instance^{(k)}}[\eventLB^{(k)}] \right|
\leq \frac{1}{24},
\]
where $\Pr_{\instance}[\cdot]$ uses the probability distribution when running the policy using input instance $\instance$.
\end{lemma}

% \begin{lemma}\label{lem:conditional prob difference}
% $\left| \Pr_{\instance}[\eventLB_1^{(k)} \given \eventLB_1^{(k)}]  - \Pr_{\instance^{(k)}}[\eventLB^{(k)} \given \eventLB_1^{(k)}] \right|
% \leq \frac{1}{24}.$ 
% \end{lemma}

Combining Lemma \ref{lem:prob difference} with inequality \eqref{eq:prob of k}, we have 
$$\Pr_{\instance^{(k)}}[\eventLB^{(k)}] \geq \frac{1}{24}.$$ 
Now, we lower bound the expected regret of the algorithm for instance $\instance^{(k)}$ at time horizon $ \bT^{1+\epsfunc{\hT}}$ as
\begin{align*}
\E_{\instance^{(k)}} \left[\regret_{\bT^{1 + \epsfunc{\hT}}}\right]
&\geq\E_{\instance^{(k)}} \left[\regret_{\bT^{1 + \epsfunc{\hT}}} \given \eventLB^{(k)}\right] \cdot \Pr_{\instance^{(k)}}[\eventLB^{(k)}] \\
&\geq (\bT^{1+\epsfunc{\hT}} - \bT) \cdot 
\frac{\frac{1}{16}\sqrt{\frac{N}{24\bT}}}
{\frac{3}{2}+\frac{1}{16}\sqrt{\frac{N}{24\bT}}}
\cdot \frac{1}{24} \\
&\geq \frac{1}{7525} \cdot \sqrt{N} \bT^{\frac{1}{2} + \epsfunc{\hT}}
\geq \frac{1}{7525} \cdot \sqrt{N}\bT^{(1 + \epsfunc{\hT})(\frac{1}{2} + \frac{\epsfunc{\hT}}{3})}, 
\end{align*}
for any $\epsfunc{\hT} \in \left[\frac{3}{\log_2 \hT}, \frac{1}{2}\right]$.
The third inequality holds because 
 $\frac{3}{2}+\frac{1}{16}\sqrt{\frac{N}{24\bT}} \leq 2$ and $\bT^{1+\epsfunc{\hT}} \geq T_0$,
and hence for $\epsfunc{\hT} \geq \frac{3}{\log_2 \hT}$, 
we have $\bT^{1+\epsfunc{\hT}} \geq \bT \cdot T_0^{\frac{\epsfunc{\hT}}{1+\epsfunc{\hT}}}
\geq 2 \bT$. 
Let $T = \bT^{1 + \epsfunc{\hT}} \in [T_0, \hT]$.
Since by assumption $\epsfunc{\cdot}$ is a non-increasing function when $T \geq T_0$, 
we have that
$\epsfunc{T} \geq \epsfunc{\hT}$, therefore
\begin{equation*}
\expect{\regret_T} 
\geq \frac{1}{7525} \cdot T^{\frac{1}{2} + \frac{\epsfunc{T}}{3}}. \qedhere
\end{equation*}
\end{proof}

Finally we need to prove Lemma \ref{lem:prob difference}. 
First we introduce the following theorem on bounding the difference of the probability for a certain event. 

\begin{theorem}[\cite{pinsker1960information}]
\label{thm:pinsker}
For any probability distribution $P, Q$ on measurable space $(X, \Sigma)$, for any event $\eventLB \in \Sigma$, 
we have 
\begin{equation*}
|P(\eventLB) - Q(\eventLB)| \leq \sqrt{\frac{1}{2}\KL(P || Q)},
\end{equation*}
where $\KL(P || Q)$ is the KL-divergence between distribution $P$ and $Q$. 
\end{theorem}

\begin{lemma}\label{lem:KL-bernoulli}
The KL divergence between two Bernoulli distributions with $p_1 = \frac{1}{3} + \Delta$ and $p_2 = \frac{1}{3}$ is 
\begin{align*}
\KL(p_1, p_2) \leq \frac{9\Delta^2}{2}
\end{align*}
\end{lemma}
\begin{proof}
The KL-divergence between two Bernoulli distributions with parameters $p_1, p_2$ is 
\begin{equation*}
\KL(p_1, p_2) = p_1 \ln \frac{p_1}{p_2} + (1-p_1)\ln \frac{1-p_1}{1-p_2}
\end{equation*}
Substituting $p_1 = \frac{1}{3} + \Delta$
and $p_2$ = $\frac{1}{3}$, 
we have 
\begin{align*}
\KL(p_1, p_2) = 
\left(\frac{1}{3} + \Delta\right) 
\ln\left(1 + 3\Delta \right)  
+ \left(\frac{2}{3} - \Delta \right) 
\ln \left(1 - \frac{3\Delta}{2} \right)
\leq \frac{9\Delta^2}{2}
\end{align*}
where the last inequality holds by $\ln(1+x) \leq x$. 
\end{proof}

\begin{proof}[Proof of Lemma \ref{lem:prob difference}]
Note that in our construction, the choice distribution at each time $t$ is a Bernoulli distribution. 
More specifically, under instance $\instance$, when item $k$ is offered to the customer, 
the probability she chooses to purchase item $k$ is $p_2 = \frac{\frac{1}{2}}{1+\frac{1}{2}} = \frac{1}{3}$,
while under instance $\instance^{(k)}$, when item $k$ is offered to the customer, 
the probability she chooses to purchase item $k$ is 
\begin{align}\label{eq:bound delta}
p_1 = \frac{\frac{1}{2}+\frac{1}{16}\sqrt{\frac{N}{24\bT}}}
{\frac{3}{2}+\frac{1}{16}\sqrt{\frac{N}{24\bT}}}
= \frac{1}{3} + \frac{\frac{1}{16}\sqrt{\frac{N}{24\bT}}}
{\frac{3}{2}+\frac{1}{16}\sqrt{\frac{N}{24\bT}}}
\leq \frac{1}{3} + \frac{1}{24}\sqrt{\frac{N}{24\bT}}.
\end{align}

In event $\eventLB^{(k)}$, the number of times item $k$ is offered is at most $\frac{48\bT}{N}$. 
The total information available to the algorithm is the set of choice distributions observed for item $k$ 
since the choice distributions for other items are the same. 
Therefore, combining \Cref{thm:pinsker}, Lemma \ref{lem:KL-bernoulli} 
and inequality \eqref{eq:bound delta}, 
we have
% Thus, the KL-divergence for choice distribution from time interval $[1, \bT]$ is at most $\bT \cdot \KL(p_1, p_2)$. 
% By  \Cref{thm:pinsker}, 
\begin{equation*}
\left| \Pr_{\instance}[\eventLB^{(k)}]  - \Pr_{\instance^{(k)}}[\eventLB^{(k)}] \right|
\leq \sqrt{\frac{1}{2} \cdot \frac{48\bT}{N} \cdot \KL(p_1, p_2)} \leq \frac{1}{24}. \qedhere
\end{equation*} 
\end{proof}

\subsection{Proof of \Cref{thm:lb-AS-loglogT}}
\label{app:lb-AS-loglogT}
The proof of \Cref{thm:lb-AS-loglogT} is similar to that of 
\Cref{thm:lb-AS-logT} except for that we divide the time periods with a different scheme.  It suffices to prove the following theorem in order to establish \Cref{thm:lb-AS-loglogT}.

\begin{theorem}\label{thm:lower known horizon}
For any $N\geq2$, $T \geq 4$, and  $M \leq \log_2\log_2 T$,  
we have that for any algorithm such that the expected number assortment switches before time horizon $T$ is
$\expect{\AS_T} \leq \frac{NM}{8}$, 
there exists an $N$-item assortment instance $\instance$  such that 
 the expected regret of the algorithm for instance $\instance$  at time horizon $T$ is 
$$\expect{\regret_T} \geq \frac{1}{7525} \cdot 
\sqrt{N}T^{\frac{1}{2(1-2^{-M})}}.$$
\end{theorem}

Before proving \Cref{thm:lower known horizon}, we first prove \Cref{thm:lb-AS-loglogT} using \Cref{thm:lower known horizon}. 
\begin{proof}[Proof of \Cref{thm:lb-AS-loglogT}]
We set $M = \lfloor \log_2(\frac{\log_2 T}{2C\log_2 \ln(NT)}) \rfloor$. 
It is easy to verify that $M$ is at most $\log_2\log_2 T$ 
for $T$ larger than a universal constant that depends on $C$. 
Now invoke \Cref{thm:lower known horizon}, and we have that for any algorithm, there exists an $N$-item assortment instance $\instance$  such that either $\expect{\regret_T} \geq\frac{1}{7525} \cdot \sqrt{NT} (\ln (NT))^C$ or 
\[
\expect{\AS_T} = \Omega\left(\frac{NM}{8}\right) = \Omega\left(N \log\log T \right) ,
\]
proving  \Cref{thm:lb-AS-loglogT}.
\end{proof}

\newcommand{\Tstage}[1]{T_{(#1)}}
\begin{proof}[Proof of \Cref{thm:lower known horizon}]
Suppose that the expected number of assortment switches by the given policy for any input instance is at most 
$\frac{NM}{8}$ before time horizon $T$, 
we will prove the theorem by showing that there exists an instance such that the expected regret incurred by the algorithm is at least $ \frac{1}{7525} \cdot 
\sqrt{N}T^{\frac{1}{2(1-2^{-M})}}$.

Consider the assortment instance $\instance = (\bm{v}, \bm{r})$, 
where $v_i = \frac{1}{2}$ and $r_i = 1$ for any $i\in [N]$.
We will let the capacity constraint  be $K=1$ for all assortment instances considered in this proof. 
By the assumption of the algorithm, the expected number of assortment switches given input instance $\instance$ is at most $\frac{M}{8}$. 
For any $j \leq M$, 
we define 
\begin{align*}
\Tstage{j} = T^{\frac{1-2^{-j}}{1-2^{-M}}}. 
\end{align*}
By definition, we have that $\Tstage{M} = T$. 
Therefore, there exists  $j$ such that $0\leq j\leq M-1$ and the expected number of assortment switches in time interval $[\Tstage{j}, \Tstage{j+1}]$ is at most $\frac{N}{8}$
since there are $M$
such disjoint intervals in range $[1, T]$.  
Let 
\begin{align*}
\eventKnown_1^{(i)} = \{ \text{item } i \text{ is not offered in time interval } [\Tstage{j}, \Tstage{j+1}] \text{ given instance } \instance\}.
\end{align*} 
Note that 
$\sum_i \Pr_{\instance}[\neg\eventKnown_1^{(i)}] \leq \frac{N}{8}+1 \leq \frac{5N}{8}$ for any $N\geq 2$, 
because the expected number of items get offered during time interval $[\Tstage{j}, \Tstage{j+1}]$ is at most the expected number of assortment switches plus~1. 
Therefore, by an averaging argument, we have that there exists a set of items $I \subseteq [N]$
such that $|I| \geq \frac{N}{4}$ 
and for any item $i\in I$, $\Pr_{\instance}[\neg\eventKnown_1^{(i)}] \leq \frac{5}{6}$.
% Hence for any item $i\in I$, 
% \begin{align}\label{eq:prob i}
% \Pr[\eventKnown_1^{(i)}] \geq \frac{1}{6}.
% \end{align} 
Define the following event
\begin{align*}
\eventKnown_2^{(i)} = \{\text{the number of times that item } i \text{ is offered in } [1, \Tstage{j}] \text{ given instance } \instance \text{ is at most } \frac{48\Tstage{j}}{N}\}.
\end{align*} 
Note that $T_1$ is at least the expected number of times an item $i\in I$ is chosen between $[1, T_1]$, 
which implies $\Tstage{j} \geq \frac{48\Tstage{j}}{N}\cdot \sum_{i\in I} \Pr_{\instance}[\neg\eventKnown_2^{(i)}]$. 
Thus there exists $k\in I$ such that $\Pr_{\instance}[\neg\eventKnown_2^{(k)}] \leq \frac{1}{12}$ since $|I| \geq \frac{N}{4}$. 
Let $\eventKnown^{(k)} = \eventKnown_1^{(k)} \cap \eventKnown_2^{(k)}$, 
we have that
\begin{align}\label{eq:prob of k known horizon}
\Pr_{\instance}[\eventKnown^{(k)}] \geq 1 - \Pr_{\instance}[\neg\eventKnown_1^{(k)}] - \Pr_{\instance}[\neg\eventKnown_2^{(k)}] \geq \frac{1}{12}.
\end{align}

Now we consider the assortment instance $\instance^{(k)} = (\bm{v}^{(k)}, \bm{r})$ 
where $v^{(k)}_{k} = \frac{1}{2} 
+ \frac{1}{16}\sqrt{\frac{N}{24\Tstage{j}}}$ 
and $v^{(k)}_j = \frac{1}{2}$ for $j\neq k$. 
Using the same proof of \Cref{lem:prob difference}, we have that
\[
\left| \Pr_{\instance}[\eventKnown^{(k)}]  - \Pr_{\instance^{(k)}}[\eventKnown^{(k)}] \right|
\leq \frac{1}{24},
\]
and combining it with inequality \eqref{eq:prob of k known horizon}, we have that
$$\Pr_{\instance^{(k)}}[\eventKnown^{(k)}] \geq \frac{1}{24}.$$ 
Now, we lower bound the expected regret of the algorithm for instance $\instance^{(k)}$ as
\begin{align*}
\E_{\instance^{(k)}} \left[\regret_{T}\right]
&\geq\E_{\instance^{(k)}} \left[\regret_{T} \given \eventKnown^{(k)}\right] \cdot \Pr_{\instance^{(k)}}[\eventKnown^{(k)}] \\
&\geq (\Tstage{j+1} - \Tstage{j}) \cdot 
\frac{\frac{1}{16}\sqrt{\frac{N}{24\Tstage{j}}}}
{\frac{3}{2}+\frac{1}{16}\sqrt{\frac{N}{24\Tstage{j}}}}
\cdot \frac{1}{24} \\
&\geq \frac{1}{7525} \cdot \Tstage{j+1} \cdot \sqrt{\frac{N}{\Tstage{j}}}
\geq \frac{1}{7525} \cdot \sqrt{N}T^{\frac{1}{2(1-2^{-M})}}, 
\end{align*}
The third inequality holds because 
$\frac{3}{2}+\frac{1}{16}\sqrt{\frac{N}{24\Tstage{j}}} \leq 2$ 
and for $j\leq M-1$, $M\leq \log_2\log_2 T$, 
we have that
\begin{equation*}
\Tstage{j+1}
= T^{\frac{1-2^{-j-1}}{1-2^{-M}}} 
\geq T^{\frac{1-2^{-j}}{1-2^{-M}}} \cdot T^{\frac{2^{-j-1}}{1-2^{-M}}}
\geq T^{\frac{1-2^{-j}}{1-2^{-M}}} \cdot T^{\frac{2^{-M}}{1-2^{-M}}}
\geq 2 T^{\frac{1-2^{-j}}{1-2^{-M}}}
= 2\Tstage{j}. \qedhere
\end{equation*}
\end{proof}

\section{$N\log T$ item switch bound for ESUCB}\label{app:ESUCB-further-improve}
In this section we show that a modification of ESUCB algorithm achieves $O(N \log T)$ item switches.

The modification is to use variables $T_i$ and $n_i$ without initializing in each $\chk(\theta_l,\theta_r,\tmax)$ sub-routine. That is, move the $T_i\gets 0, n_i\gets 0$ statement to the initialize phase of Algorithm~\ref{alg:trisection}. Note that $n_i/T_i$ is still an unbiased estimation of $v_i$, and only concentrates better. As a result, the regret analysis applies directly.

Regarding the number of item switches, since the value of $T_i$ and $n_i$ are not initialized in $\chk$ procedure, number of updates in value $\hat{v}_i$ is bounded by $\log T$ during the execution of ESUCB algorithm, instead of $\log^2 T$ when initialization is executed in $\chk$. Therefore we can give a better upper bound on the item switch of ESUCB algorithm. The following theorem shows the item switch bound of modified ESUCB algorithm.
\begin{theorem}
The number of item switches incurred by ESUCB algorithm is bounded by $O(N\log T).$
\end{theorem}
\begin{proof}
Recall that $S_\ell$ is calculated by $S_\ell=\arg\max_{S\in [N],|S|\le K}\left(\sum_{i\in S}\hat{v}_i(r_i-\theta)\right)$ for some $\theta$ (Line~\ref{line:erm-with-fixed-theta-0} and Line~\ref{line:erm-with-fixed-theta} of Algorithm~\ref{alg:check}). Observe that the value of $b$ in Algorithm~\ref{alg:check} can only be switched once in an invocation. Therefore the number of switches in value $\theta$ is upper bounded by $O(\log T).$ The item number of item switch introduced by the change of $\theta$ is then bounded by $O(N\log T).$ Now, consider an consecutive time steps where $\theta$ is unchanged. We only need to show that for fixed any $\theta$, and $S_{\ell}' = \mathop{\arg\max}_{S \subseteq [N], |S| \leq K}\left(\sum_{i\in S}\hv_i(r_i-\theta)\right)$, it holds that (assuming that there are $L$ epochs)
\begin{align}\label{eq:lem-item-switch-check-goal-app}
%\textstyle
\sum_{\ell = 1}^{L - 1} |S_{\ell}' \oplus S_{\ell + 1}'| \lesssim N \log T .
\end{align}
Suppose that there are $n_{\ell}$ items whose UCB values are updated after the $\ell$-th epoch. We claim that $|S_{\ell} \oplus S_{\ell + 1}| \leq n_{\ell}$. This is simply because $S_{\ell}$ corresponds to the items $i \in [N]$ such that $\htv_i (r_i - \theta)$ is positive and among the $K$ largest ones (and thanks to the tie breaking rule). Therefore, any update to a single $\htv_i$ will incur at most one item switch to $S_{\ell}$, and $n_{\ell}$ updates will incur at most $n_{\ell}$ item switches. Now, \eqref{eq:lem-item-switch-check-goal-app} is established because
$\sum_{\ell = 1}^{L - 1} |S_{\ell}' \oplus S_{\ell + 1}'|  \leq \sum_{\ell = 1}^{L-1} n_\ell \lesssim N \log T$,
where the second inequality is due to the deferred update rule for the UCB values.
\end{proof}

\end{document}